\documentclass[11pt]{article}

\usepackage[utf8]{inputenc}
\usepackage[margin=0.95in]{geometry} %default geometry package already reduces margins
\usepackage{amsmath, amssymb, amsfonts, amsthm}
\usepackage{mathrsfs}
\usepackage{url}
\usepackage{txfonts, cite}

\numberwithin{equation}{section}
\usepackage[pagewise]{lineno}%\linenumbers %vorgeschrieben für Einreichung

\usepackage{mathrsfs}
\usepackage{mathtools}
\mathtoolsset{showonlyrefs}
\usepackage{orcidlink}
\theoremstyle{definition}

\newtheorem{example}{Example}
\newtheorem{remark}{Remark}

\theoremstyle{theorem}
\newtheorem{theorem}{Theorem}
\newtheorem*{theorem-non}{Theorem}
\newtheorem{proposition}{Proposition}

\newtheorem{corollary}{Corollary}

\newcommand{\cX}{\mathcal{X}}
\newcommand{\cA}{\mathcal{A}}
\newcommand{\cS}{\mathcal{S}}
\newcommand{\cT}{\mathcal{T}}

\newcommand{\cbT}{\mathcal{S}}
\newcommand{\sP}{\mathscr{P}}
\newcommand{\bR}{\mathbb{R}}
\newcommand{\bE}{\mathbb{E}}

\newcommand{\bN}{\mathbb{N}}
\newcommand{\bS}{\mathbb{S}}

\newcommand{\cL}{\mathcal{L}}
\newcommand{\cH}{\mathcal{H}}

\newcommand{\St}{S^{(t)}}
\newcommand{\At}{A^{(t)}}
\newcommand{\Rt}{R^{(t)}}

\newcommand{\bbG}{\textbf{G}}
\newcommand{\bbX}{\textbf{X}}

\newcommand{\bbw}{\textbf{w}}

\newcommand{\bbe}{\textbf{e}}

\newcommand{\bbx}{\textbf{x}}
\newcommand{\bbJ}{\textbf{J}}
\newcommand{\bbR}{\textbf{R}}

\newcommand{\bbA}{\textbf{A}}
\newcommand{\bbB}{\textbf{B}}

\newcommand{\ed}{\overset{d}{=}}
\newcommand{\bP}{\mathbb{P}}

\newcommand{\pr}{\mathrm{pr}}

\DeclareMathOperator*{\sgn}{\mathbin{sgn}}

\DeclareMathOperator*{\argmax}{\mathbin{argmax}}
\DeclareMathOperator*{\geometric}{\mathbin{Geo}}

\begin{document}

\title{On solutions of the distributional Bellman equation}
\date{\today}
\author{Julian Gerstenberg, Ralph Neininger and Denis Spiegel \\ Institute for Mathematics\\
Goethe University Frankfurt\\ Germany\\
{\tt \{gerstenb,neiningr,spiegel\}@math.uni-frankfurt.de}}

\maketitle

\abstract{
	In distributional reinforcement learning not only expected returns but the complete return distributions of a policy are taken into account. The return distribution for a fixed policy is given as the solution of an associated distributional Bellman equation.
	In this note we consider general distributional Bellman equations and study existence and uniqueness of their solutions as well as tail properties of return distributions. We give necessary and sufficient conditions for existence and uniqueness of return distributions and identify cases of regular variation.
	
	We link distributional Bellman equations to multivariate affine distributional equations. We show that any solution of a distributional Bellman equation can be obtained as the vector of marginal laws of a solution to a multivariate affine distributional equation. This makes the general theory of such equations applicable to the distributional reinforcement learning setting.}
\\\\
\textbf{Keywords:} distributional reinforcement learning; distributional Bellman equation; random difference equation; perpetuity; Markov decision process; regular variation; machine learning
\\
\textbf{AMS subject classifications:} Primary: 60E05, 60H25, Secondary: 68T05, 90C40

\section{Introduction}\label{sec:introduction}

The objective in reinforcement learning (RL) is to teach an agent that sequentially interacts with an environment to choose 'good' actions. For each action the agent receives an immediate real-valued reward. The rewards are accumulated over time resulting in the so-called \emph{return}, which describes the overall performance of the agent. Randomness may be involved at all levels of this interaction: in choosing actions, the environment reacting to actions and/or in the rewards received by the agent. Hence, the return is to be considered random as well. In more classical approaches to RL problems the randomness is averaged out and only the expected return is considered when evaluating the performance of an agent. In \cite{bellemare2017distributional}, not only the expectation but the complete distribution of the return was considered, introducing what is now known as \emph{distributional} RL, see \cite{bdr2022}.

Mathematically, a RL problem is typically modeled by a \emph{Markov decision process} (MDP), that is as a particular type of a discrete-time stochastic control problem. An overview of the classical MDP theory and its applications is presented in \cite{puterman2014markov}. For more details on distributional RL using notations similar to here we refer to \cite{rowland2018analysis,bdr2022}.

For any measurable space $\cX$ we write $\sP(\cX)$ for the set of probability distributions on $\cX$. The distribution (law) of a $\cX$-valued random variable (RV) $X$ is denoted by $\cL(X) = \bP[X\in\cdot~]\in\sP(\cX)$. We also write short $X\sim \mu$, if $\cL(X) =\mu\in\sP(\cX)$. In case $\cX$ is countable, discrete distributions $\nu\in\sP(\cX)$ can be identified with functions $\nu:\cX\rightarrow[0,1]$ satisfying $\sum_{x\in\cX}\nu(x)=1$. For a random variable $X$ and an event $A$ with $\bP[A]>0$ we write $\cL(X|A)\in\sP(\cX)$ for the conditional distribution of $X$ given $A$.\\

Let $\cS$ and $\cA$ be non-empty finite sets. A MDP on states $\cS$ and actions $\cA$ is a function

\begin{equation}\label{eq:mdp}
\rho:\cS\times\cA\rightarrow\sP(\cS\times\bR),~~(s,a)\mapsto \rho_{(s,a)}.
\end{equation}

The interpretation is as follows: $\cS$ is the set of states an environment can occupy and $\cA$ the set of possible actions the agent can perform. If in state $s\in\cS$ the agent performs action $a\in\cA$ the environment reacts with a (possibly random) successor state $S$ together with a (possibly random) real-valued reward $R$ having joint distribution $(S,R)\sim \rho_{(s,a)}\in\sP(\cS\times\bR)$. How an agent chooses its actions can be modeled by a (stationary) \emph{policy}

\begin{equation}\label{eq:policy}
\pi:\cS\rightarrow \sP(\cA),~~s\mapsto \pi_s.
\end{equation}

An agent that is in state $s$ and acts according to policy $\pi$ chooses a (possibly random) action $A$ distributed as $A\sim\pi_s\in\sP(\cA)$.

Suppose an agent acting according to $\pi$. Starting at time $t=0$ from a possibly random state $S^{(0)}$ the following dynamic is defined inductively: at time $t\in\bN_0$ the agent finds itself in state $\St$, chooses action $\At\sim\pi_{\St}$ (independent of the past given $\St$) and the environment reacts with the next state $S^{(t+1)}$ and immediate reward $\Rt$ jointly distributed as $(S^{(t+1)},\Rt)\sim\rho_{(\St,\At)}$ (independent of the past given $(\St,\At)$). The resulting stochastic process
\begin{equation}\label{eq:mdp-process}
	(\St,\At,\Rt)_{t\in\bN_0}
\end{equation}
is called \emph{Markov reward process} in \cite{bdr2022} or \emph{full MDP} in \cite{rowland2018analysis}. To emphasize that the distribution of the stochastic process \eqref{eq:mdp-process} depends on $\pi$, we write $\bP^{\pi}$ instead of $\bP$.

A suitable way to judge the overall performance of the agent is to choose a \emph{discount factor} $\gamma\in(0,1)$ and to consider the discounted accumulated rewards, the so-called \emph{return}:
\begin{equation}\label{eq:return}
G = R^{(0)} + \gamma R^{(1)} + \gamma^2 R^{(2)} + \cdots = \sum_{t=0}^{\infty}\gamma^t\Rt.
\end{equation}
The return $G$ is defined as an infinite series of random variables and its existence as a $\bR$-valued random variable is not automatically guaranteed, for example if rewards are (extremely) heavy-tailed. Theorem~\ref{thm:existence-drl} below provides a complete characterization of when the return $G$ exists almost surely as a $\bR$-valued random variable starting from any state $S^{(0)}=s$. The relation of Theorem \ref{thm:existence-drl} to earlier results is discussed in Remarks \ref{rem_als_buck}, \ref{first_remark} and \ref{rem:weak}. In various applications, heavy-tailed reward distributions are not encountered. For instance, if rewards are obtained as a deterministic function of states and actions, the rewards become bounded, and the existence of the return $G$ easily follows by a geometric series argument. However, unbounded and even heavy-tailed reward distributions are of interest in various applications: fields in which RL approaches have been considered and heavy-tailed distributions play a crucial role include insurance and finance. See \cite{krasheninnikova2019reinforcement} and \cite{kolm2020modern} for RL approaches to pricing and trading and \cite{ekm} for the role of heavy-tailed distributions in that field. Another research area in which heavy-tailed reward distributions have been considered recently is the study of multi-armed bandit models, which corresponds to an MDP with only a single state, $|\cS|=1$. In \cite{pmlr-v130-zhuang21a} the authors report on several studies in this direction, including linear bandits \cite{pmlr-v48-medina16}, \cite{yu2018pure}, pure exploration \cite{shao2018almost}, Lipschitz-bandits \cite{lu2019optimal}, Bayesian optimization \cite{ray2019bayesian} and Thompson sampling \cite{dubey2019thompson}. Notably, in the latter work a particular emphasis is placed on $\alpha$-stable reward distributions, which is an assumption also covered by our analysis, see Theorem~\ref{thm:regulary_varying_tails} and Example~\ref{ex:regvar}. It is worth noting that all of these works consider finite horizon non-discounted returns rather than infinite discounted returns. Our analysis offers a theoretical justification for the inclusion of heavy-tailed reward distributions in infinite discounted reward scenarios.

\begin{remark}
    The case $\gamma=0$ is sometimes considered in RL literature, resulting in $G=R^{(0)}$ and trivializing many of the research questions we investigate here. Although we do not discuss this case further, it should be noted that several parts of our results remain applicable even when $\gamma=0$.
\end{remark}

Suppose the return $G$ exists as a $\bR$-valued RV starting from any state $s$. Of interest in \emph{distributional policy evaluation} are the distributions of the return starting from given states. The \emph{state(-action) return distributions} are defined by

\begin{equation}
\begin{aligned}[b]
\eta^{\pi}_s &= \bP^{\pi}[G\in\cdot~|S^{(0)}=s],~~~s\in\cS,\\
\eta^{\pi}_{(s,a)} &= \bP^{\pi}[G\in\cdot~|S^{(0)}=s, A^{(0)}=a],~~~(s,a)\in\cS\times\cA.
\end{aligned}
\label{eq:value_dist}
\end{equation}

In \cite{bdr2022} the collection $\eta^{\pi} = (\eta^{\pi}_s)_{s\in\cS}$ is called \emph{return distribution function}. The state(-action) return distributions can be found as solutions to the \emph{distributional Bellman equations}, which we now explain. For $r\in\bR, \gamma\in(0,1)$ let $f_{r,\gamma}:\sP(\bR)\rightarrow\sP(\bR)$ be the map that sends $\nu = \cL(X)$ to $f_{r,\gamma}(\nu) = \cL(r + \gamma X)$. Using the recursive structure of the return, $G = R^{(0)} + \gamma G'$ with $G' = \sum_{t=0}^{\infty}\gamma^t R^{(t+1)}$, and Markov properties of \eqref{eq:mdp-process} it is seen that state and state-action return distributions are related by

\begin{equation}
\begin{aligned}[b]
\eta^{\pi}_s &= \sum_{a\in\cA}\pi_s(a)\eta^{\pi}_{(s,a)},~~~s\in\cS,\\
\eta^{\pi}_{(s,a)} &= \int_{\cS\times\bR}f_{r,\gamma}\big(\eta^{\pi}_{s'}\big)d\rho_{(s,a)}(s',r),~~~(s,a)\in\cS\times\cA.
\end{aligned}
\label{eq:state_vs_stateaction}
\end{equation}

Substituting the formulas in one another yields the distributional Bellman equations, which come in two forms: one for states and one for state-actions:

\begin{equation}
\begin{aligned}[b]
	\eta^{\pi}_s &= \sum_{a\in\cA}\int_{\cS\times\bR}\pi_s(a)f_{r,\gamma}\big(\eta^{\pi}_{s'}\big)d\rho_{(s,a)}(s',r),~~~s\in\cS,\\
	\eta^{\pi}_{(s,a)} &= \int_{\cS\times\bR}\sum_{a'\in\cA}\pi_{s'}(a')f_{r,\gamma}\big(\eta^{\pi}_{(s',a')}\big)d\rho_{(s,a)}(s',r),~~~(s,a)\in\cS\times\cA.
\end{aligned}
\label{eq:dbe}
\end{equation}

The distributional Bellman equations are a system of $|\cS|$ resp. $|\cS\times\cA|$ one-dimensional distributional equations in $\bR$. The right hand side of the equations can be used to introduce the \emph{distributional Bellman operator}, which for the state return distributions is defined as:

\begin{equation}
\begin{aligned}[b]
	\cT^{\pi}:\sP(\bR)^{\cS}\longrightarrow\sP(\bR)^{\cS},~~~\eta=(\eta_s)_{s\in\cS}&\longmapsto \cT^{\pi}(\eta) = \Big(\cT^{\pi}_s(\eta)\Big)_{s\in\cS}  \\
	\text{with $s$-th component function}~~&\cT^{\pi}_s(\eta) = \sum_{a\in\cA}\int_{\cS\times\bR}\pi_s(a)f_{r,\gamma}\big(\eta_{s'}\big)d\rho_{(s,a)}(s',r).
\end{aligned}
\label{eq:dbop}
\end{equation}

No assumption on the MDP $\rho$ nor the policy $\pi$ is needed to define \eqref{eq:dbop} and the relation to the state return distributions is as follows: if the return $G$, starting from any state $s$, converges almost surely in $\bR$ then the return distribution function $\eta^{\pi} = (\eta^{\pi}_s)_{s\in\cS}$ is a fixed point of $\cT^{\pi}$, that is $\eta^{\pi} = \cT^{\pi}(\eta^{\pi})$. We explore this connection in more detail.

We write $\log^+(x) = \log(x)$ for $x\geq 1$ and $\log^+(x)=0$ for $x<1$. The definition of 'essential state' is given before Theorem~\ref{thm:existence}. One result of this note is the following:

\begin{theorem}\label{thm:existence-drl}
	The following are equivalent:
	\begin{itemize}
		\item[(i)] $\cT^{\pi}$ has a fixed point $\eta\in\sP(\bR)^{\cS}$,
		\item[(ii)] $G=\sum_{t=0}^{\infty}\gamma^tR^{(t)}$ converges $\bP^{\pi}[~\cdot~|S^{(0)}=s]$-almost surely in $\bR$ for every initial state $s\in\cS$,
		\item[(iii)] $\int_{\cS\times\bR}\log^+(|r|)d\rho_{(s,a)}(s',r)<\infty$ for every pair $(s,a)\in\cS\times\cA$ that is essential with respect to the law of the Markov chain $(S^{(t)},A^{(t)})_{t\in\bN_0}$ under $\bP^{\pi}$,
		\item[(iv)] For every $\nu\in\sP(\bR)^{\cS}$ as $n\rightarrow\infty$ the sequence $\big(\cT^{\pi}\big)^{\circ n}(\nu)$ of $n$-th iterations of $\cT^{\pi}$ converges weakly in the product space $\sP(\bR)^{\cS}$.
	\end{itemize}
	If these hold the fixed point of $\cT^{\pi}$ is unique, given by the return distribution function $\eta^{\pi}$ and\\ $\big(\cT^{\pi}\big)^{\circ n}(\nu)\rightarrow \eta^{\pi}$ weakly as $n\rightarrow\infty$ for every $\nu\in\sP(\bR)^{\cS}$.
\end{theorem}

A theorem that reads analogously can be formulated for the distributional Bellman operator $\cT^{\pi}:\sP(\bR)^{\cS\times\cA}\rightarrow\sP(\bR)^{\cS\times\cA}$ defined in terms of state-actions. We cover both cases simultaneously by introducing simplified notations in Section~\ref{sec:notations} that allow to analyze the distributional Bellman operator, its fixed point equations and its solutions more conveniently. In Section~\ref{sec:results} we present our main results in these notations, Theorem~\ref{thm:existence-drl} will follow from the more general Theorem~\ref{thm:existence} presented there. For connections to a model and results in \cite{als_buck} see Remark \ref{rem_als_buck}.
\\
Besides giving necessary and sufficient conditions for the existence of solutions to the distributional Bellman equations, we also study properties of their solutions, that is of the return distribution function $\eta^{\pi}$ (and also state-action return distributions). We consider tail probability asymptotics $\eta^{\pi}_s((-\infty,-x))$ and $\eta^{\pi}_s((x,\infty))$ as $x\rightarrow\infty$, in Theorem~\ref{thm:property_transfer} we identify cases of exponential decay and existence of $p$-th moments and in Theorem~\ref{thm:regulary_varying_tails} we cover regular variation. A possible application of the latter is in Distributional Dynamic Programming (see Chapter~5 in \cite{bdr2022}), see Remark~\ref{rem:distributional-policy-evaluation} of the present note.\\

Building upon our simplified notation we explore the connection of the distributional Bellman equations to multivariate distributional equations of the form $\bbX\ed\bbA\bbX+\bbB$ in Section~\ref{sec:multivariate-fixed-point-equations}. This connection seems to have not been noticed in the literature so far and can be used to apply available results in that field to the distributional RL setting -- we do that when proving Theorem~\ref{thm:regulary_varying_tails} by applying results presented in \cite{buraczewski2016stochastic}.\\

Before we switch to the simplified notation and present the main results, we shortly present the \emph{ordinary} Bellman equations for convenience.

\subsection{The ordinary Bellman equations}\label{sec:be}

The following is well-known and the standard setting in much of the distributional RL literature: if reward distributions $\rho_{(s,a)}(\cS\times\cdot)\in\sP(\bR)$ have finite expectations for all $(s,a)$, the return converges almost surely and has finite expectations as well; this also follows from Theorems~\ref{thm:existence} and \ref{thm:property_transfer} presented in this note. In many classical approaches to RL only expected returns are used for policy evaluation, the \emph{state values} and \emph{state-action values} are defined by

\begin{equation}
\begin{aligned}[b]
	v^{\pi}_s & = \bE^{\pi}[G|S^{(0)}=s] = \int_{\bR} g~d\eta^{\pi}_s(g), \\
	q^{\pi}_{(s,a)} &= \bE^{\pi}[G|S^{(0)}=s,A^{(0)}=a] = \int_{\bR} g~d\eta^{\pi}_{(s,a)}(g).
\end{aligned}
\label{eq:value}
\end{equation}

The collection $v^{\pi} = (v^{\pi}_s)_{s\in\cS}$ is called state value function and $q^{\pi}$ state-action value function. They can be found as solutions to the (ordinary) \emph{Bellman equations}, which can be derived from their distributional counterparts \eqref{eq:dbe} using linearity of expectation. For action $a$ in state $s$ let $r_{(s,a)}\in\bR$ be the expected reward and $p_{(s,a)}\in\sP(\cS)$ the distribution of the next state, that is

\begin{equation}
\begin{aligned}[b]
	r_{(s,a)} &= \int r~d\rho_{(s,a)}(s',r) = \bE^{\pi}[R^{(t)}|S^{(t)}=s,A^{(t)}=a],\\
	p_{(s,a)} &= \rho_{(s,a)}(\cdot\times\bR) = \bP^{\pi}[S^{(t+1)}\in\cdot~|S^{(t)}=s,A^{(t)}=a].
\end{aligned}
\label{eq:mdp-reduced}
\end{equation}

The ordinary Bellman equations are a system of $|\cS|$ (resp. $|\cS\times\cA|$) linear equations in the same number of unknowns and read as follows

\begin{equation}
\begin{aligned}[b]
	v^{\pi}_s &= \sum_{a\in\cA}\pi_s(a)\Big[r_{(s,a)} + \gamma \sum_{s'\in\cS}p_{(s,a)}(s')v^{\pi}_{s'}\Big],~~s\in\cS,\\
	q^{\pi}_{(s,a)} &= r_{(s,a)} + \gamma\sum_{s'\in\cS}p_{(s,a)}(s')\sum_{a'\in\cA}\pi_{s'}(a') q^{\pi}_{(s',a')},~~(s,a)\in\cS\times\cA.
\end{aligned}
\label{eq:be}
\end{equation}

As in the distributional setting, one can use the right-hand side of these equations to define the (ordinary) Bellman operators ($\bR^{\cS}\rightarrow\bR^{\cS}$ for states) such that \eqref{eq:be} are equivalent to the associated fixed point equation of these operators.\\

When judging a policy based on expected returns, $\pi$ is considered to be at least as good as some other policy $\tilde \pi$ if $v^{\pi}\geq v^{\tilde \pi}$ pointwise at each $s$. Classical MDP theory shows that there always exists a policy at least as good as every other policy. Being able to calculate expected returns is not just of interest for evaluating $\pi$, but also to improve it: the policy $\pi'$ defined by $\pi'_s = \delta_{\argmax_{a\in\cA}q^{\pi}_{(s,a)}}$ is at least as good as $\pi$. This leads to an iterative algorithm known as \emph{policy iteration}: starting from an initial policy $\pi^{(0)}$ and letting $\pi^{(k+1)} = \big(\pi^{(k)}\big)'$ it can be shown that $v^{\pi^{(k)}}$ converges as $k\rightarrow\infty$ to the state value function of an optimal policy.

\subsection{Simplified Notations}\label{sec:notations}

We introduce simplified notations and define the distributional Bellman operator in this setting. Let $d\in\bN$ and $(R_1,J_1),\dots,(R_d,J_d)$ be random pairs such that each $R_i$ is real-valued and $J_i$ is discrete taking values in $[d]=\{1,\dots,d\}$. For each $i\in[d]$ it is $\cL(R_i,J_i)\in\sP(\bR\times[d])$ the joint distribution of the random pair $(R_i,J_i)$. The random pairs $(R_1,J_1),\dots,(R_d,J_d)$ together with some constant $\gamma\in(0,1)$ define (a version of) the \emph{distributional Bellman operator} as follows:
\begin{equation*}
\cT:\sP(\bR)^d\longrightarrow\sP(\bR)^d,~~~\eta = \left(\begin{matrix}\eta_1\\\vdots\\\eta_d\end{matrix}\right) \longmapsto \cT(\eta) = \left(\begin{matrix}\cT_1(\eta)\\\vdots\\\cT_d(\eta)\end{matrix}\right),
\end{equation*}
the $i$-th coordinate function $\cT_i:\sP(\bR)^d\rightarrow\sP(\bR)$ being
\begin{equation}\label{eq:operator}
\cT_i(\eta) = \cT_i(\eta_1,\dots,\eta_d) = \cL(R_i + \gamma G_{J_i}),
\end{equation}
where on the right hand side $\cL(G_j)=\eta_j$ for each $j\in[d]$ and $G_1,\dots, G_d, (R_i, J_i)$ are independent.

\begin{example}[State return]\label{example:statereturn}
	Let $\rho$ be a MDP and $\pi$ a stationary policy. Let $d=|\cS|$ and choose an arbitrary enumeration $i=s$ to identify $[d]$ with $\cS$. Let $i=s, j=s'$ and $B\subseteq\bR$ measurable. Consider the joint distribution
	\begin{align*}
	\bP[R_i\in B, J_i=j] &= \bP^{\pi}[R^{(t)}\in B, S^{(t+1)}=s'|S^{(t)}=s]\\
	&= \sum_{a\in\cA}\pi_s(a)\rho_{(s,a)}(\{s'\}\times B).
	\end{align*}
	In this case $\cT$ in \eqref{eq:operator} is the same as $\cT^\pi$ in \eqref{eq:dbop}.
\end{example}

\begin{example}[State-action return]
	Let $\rho$ be a MDP and $\pi$ a stationary policy. Let $d=|\cS\times\cA|$ and choose an arbitrary enumeration $i=(s,a)$ to identify $[d]$ with $\cS\times\cA$. Let $i=(s,a), j=(s',a')$ and $B\subseteq\bR$ measurable. Consider the joint distribution
	\begin{align*}
	\bP[R_i\in B, J_i=j] &= \bP^{\pi}[R^{(t)}\in B, S^{(t+1)}=s', A^{(t+1)}=a'|S^{(t)}=s, A^{(t)}=a]\\
	&= \rho_{(s,a)}(\{s'\}\times B)\pi_{s'}(a').
	\end{align*}
	In this case the distributions $\cL(R_i,J_i), i\in[d]$ completely encode both $\rho$ and $\pi$ and, provided existence, state-action return distributions $(\eta^{\pi}_{(s,a)})_{(s,a)\in\cS\times\cA}$ are fixed points of $\cT$.
\end{example}

The connection to distributional RL shows that it is of interest to understand the fixed points of $\cT$, that is probability distributions $\eta\in\sP(\bR)^d$ satisfying
\begin{equation}\label{eq:fixed_point_operator}
\eta = \cT(\eta).
\end{equation}
Let $G_1,\dots,G_d$ be real-valued random variables with $\eta_i = \cL(G_i)$. Then \eqref{eq:fixed_point_operator} is equivalent to the distributional Bellman equations, which is a system of $d$ one-dimensional distributional equations
\begin{equation}\label{eq:fixed_point_system}
	G_i\ed R_i + \gamma G_{J_i},~~i\in[d],
\end{equation}
where $\ed$ (also denoted by $=_d$) denotes equality in distribution and in the $i$-th equation $G_1,\dots,G_d, (R_i,J_i)$ are independent, $i\in[d]$. For $i,j\in[d]$ let
\begin{equation*}
	p_{ij} = \bP[J_i=j]~~\text{and}~~\mu_{ij} = \cL(R_i|J_i=j),
\end{equation*}
where in case $p_{ij}=0$ we set $\mu_{ij} = \delta_0$, the Dirac measure in $0$. It holds $\bP[R_i\in \cdot, J_i=j] = p_{ij} \mu_{ij}(\cdot)$. Let $F_i(x) = \bP[G_i\leq x]$ be the cumulative distribution function (cdf) of real-valued random variable $G_i$. A third equivalent formulation of the distributional Bellman equations \eqref{eq:fixed_point_system} in terms of cdf's is
\begin{equation}\label{eq:fixed_point_system_cdfs}
	F_i(\cdot) = \sum_{j=1}^d p_{ij} \int_{-\infty}^{\infty} F_j\Big(\frac{\cdot~-~r}{\gamma}\Big)d\mu_{ij}(r),~~i\in[d].
\end{equation}
The distributional Bellman equations in terms of cdfs appeared in the literature before, for instance, in \cite{morimura2010nonparametric} or, in a more specific form, in \cite{chung1987discounted}.

We study the following aspects regarding solutions to the distributional Bellman equations, i.e., fixed points of $\cT$:
\begin{itemize}
	\item Existence and uniqueness for which conditions being necessary and sufficient are given in Theorem~\ref{thm:existence}, see also Remark \ref{first_remark}.
	\item Tail probability asymptotics. That is, if $\eta$ is a fixed point of $\cT$ with $i$-th component $\eta_i = \cL(G_i)$ we study the asymptotic behavior of $\bP[G_i>x]$ and $\bP[G_i<-x]$ as $x\rightarrow\infty$. Depending on properties of the distributions $\cL(R_1,J_1),\dots,\cL(R_d,J_d)$ we identify cases of
	\begin{itemize}
		\item exponential decay and existence of $p$-th moments, see Theorem~\ref{thm:property_transfer}
		\item regular variation, see Theorem~\ref{thm:regulary_varying_tails}.
	\end{itemize}
\end{itemize}

In Section~\ref{sec:multivariate-fixed-point-equations} we explain how our findings relate to the area of multivariate distributional fixed point equations of the form
\begin{equation}\label{eq:mult}
	\bbX \ed \bbA\bbX + \bbB,
\end{equation}
in which $\bbX, \bbB$ are random $d$-dimensional vectors, $\bbA$ a random $d\times d$ matrix and $\bbX$ and $(\bbA,\bbB)$ are independent. Such equations arise in many different applications, for example in the context of perpetuities and random difference equations. In the above notation, the distributional Bellman equations \eqref{eq:fixed_point_system} are a system of \emph{$d$ one-dimensional distributional equations} which is less restrictive as a \emph{single distributional equation in $\bR^d$} such as \eqref{eq:mult}, but more complicated then a single one-dimensional distributional equation. However, using Theorem~\ref{thm:existence} we show how general results from the multivariate setting become applicable to study the distributional Bellman equations, see Corollary~\ref{cor:multivariate}.

In case $d=1$ the random pairs $(R_1,J_1),\dots,(R_d,J_d)$ reduce to a single real-valued random variable $(R_1, J_1)$ with $J_1\equiv 1$ and the distributional Bellman equations \eqref{eq:fixed_point_system} reduce to a single distributional equation in $\bR$:
\begin{equation}\label{eq:d=1}
	G_1 \ed \gamma G_1 + R_1,
\end{equation}
with $G_1$  independent of $R_1$. This is a simple special case of a one-dimensional distributional equation $X =_d AX+B$, also called perpetuity equation, for which the existence and uniqueness of solutions as well as their properties have been studied in great detail, see, e.g., \cite{vervaat}, \cite{golgru} or \cite{golmal}. A comprehensive discussion is given in Chapter~2 in \cite{buraczewski2016stochastic}.

\section{Results}\label{sec:results}

Let $\gamma\in(0,1)$ and $(R_1,J_1),\dots,(R_d,J_d)$ be the $\bR\times[d]$-valued random pairs defining the Bellman operator $\cT:\sP(\bR)^d\rightarrow\sP(\bR)^d$ as in \eqref{eq:operator}. Some results can be formulated and proven within a special construction based on random variables
\begin{equation}\label{eq:construction}
	(I_t)_{t\in\bN_0},~~(R_{ijt})_{(i,j)\in[d]^2, t\in\bN_0},
\end{equation}
in which 
\begin{itemize}
    \item $(I_t)_{t\in\bN_0}$ is a $[d]$-valued homogeneous Markov chain having transition probabilities $(p_{ij})_{(i,j)\in[d]^2}$ and a uniform starting distribution, that is for any path $i_0,i_1,\dots,i_T\in[d]$ it holds that
    $$\bP[I_0=i_0,\dots,I_T=i_T] = d^{-1}p_{i_0i_1}p_{i_1i_2}\cdots p_{i_{T-1}i_T},$$
    \item $(R_{ijt})_{(i,j)\in[d]^2, t\in\bN_0}$ forms an array of independent real-valued random variables with
$$\cL(R_{ijt}) = \cL(R_i|J_i=j) = \mu_{ij},$$
in particular, for each pair $(i,j)$, the sequence $(R_{ijt})_{t\in\bN_0}$ is independent identically distributed (iid) with distribution $\mu_{ij}$,
    \item $(I_t)_{t\in\bN_0}$ and $(R_{ijt})_{(i,j)\in[d]^2, t\in\bN_0}$ are independent.
\end{itemize}
From now on we omit quantification of indices in array-like quantities once the index-ranges have been introduced.

Using RVs \eqref{eq:construction} an explicit description of the $n$-th iteration of $\cT$ defined inductively as ${\cT^{\circ n} = \cT^{\circ (n-1)}\circ \cT}$ can be given: the $i$-th component of $\cT^{\circ n}$ is the map $\cT^{\circ n}_i:\sP(\bR)^d\rightarrow\sP(\bR)$ defined by
\begin{equation}\label{eq:composition}
	\cT^{\circ n}_i(\eta_1,\dots,\eta_d) = \cL\Big(\sum\nolimits_{t=0}^{n-1}\gamma^t R_{I_t, I_{t+1}, t}~+~\gamma^n G_{I_n}~\big|~I_0=i~\Big),
\end{equation}
with $\cL(G_j)=\eta_j$ and $G_1,\dots,G_d, (R_{ijt})_{ijt}, (I_t)_t$ independent.\\

To introduce some basic definitions from Markov chain theory let $i,j\in[d]$ be states. We write $i\rightarrow j$ if there exists $t\in\bN_0$ with $\bP[I_t=j|I_0=i]>0$. State $i$ is \emph{essential} if for every state $j$ the implication $i\rightarrow j\Rightarrow j\rightarrow i$ holds. Since the state space $[d]$ is finite, a state $i$ is essential if and only if it is \emph{recurrent}, that is $\bP[I_t=i~\text{for some}~t>0|I_0=i]=1$, which implies $\bP[I_t=i~\text{for $\infty$-many $t$}|I_0=i]=1$.

\begin{theorem}\label{thm:existence}
	The following conditions are equivalent:
	\begin{enumerate}
		\item[(i)] $\cT$ has a (unique) fixed point,
		\item[(ii)] $\bE[\log^+|R_i|]<\infty$ for all essential $i$,
		\item[(iii)] The infinite series $G = \sum_{t=0}^{\infty}\gamma^tR_{I_t,I_{t+1},t}$ is almost surely (absolutely) convergent,
		\item[(iv)] $\cT^{\circ n}_i(\nu)$ converges weakly as $n\rightarrow\infty$ for each $i\in[d]$ for some (all) $\nu\in\sP(\bR)^d$.
	\end{enumerate}
	If one and hence all of (i)--(iv) hold, then the unique fixed point $\eta\in\sP(\bR)^d$ of $\cT$ has $i$-th component
	\begin{equation}
		\eta_i = \cL(G|I_0=i) = \cL\Big(\sum\nolimits_{t=0}^{\infty}\gamma^tR_{I_t,I_{t+1},t}~\big|~I_0=i\Big).
	\end{equation}
\end{theorem}

\begin{remark}\label{rem_als_buck}
In \cite{als_buck} a model of iterations of random univariate affine linear maps within a Markovian environment is introduced and its stability is studied. This model is similar to our version of the distributional Bellman operator in \eqref{eq:operator}. In \cite{als_buck} the Markov chain is allowed to have countable state space but required to be ergodic whereas the present Markov chain in the distributional RL setting has finite spate space and no further restrictions. The results in Section~3 of \cite{als_buck} partly imply claims of the present Theorem \ref{thm:existence} with similar underlying techniques based on \cite{vervaat,golmal}.
\end{remark}

Note that for a real-valued random variable $X$ the following much stronger properties each imply $\bE[\log^+|X|]<\infty$:
\begin{enumerate}
	\item $|X|$ is bounded: $|X|\leq K$ for some constant $K\geq 0$,
	\item $|X|$ has a finite exponential moment: $\bE[\exp(\beta|X|)]<\infty$ for some $\beta>0$,
	\item $|X|$ has a finite $p$-th moment: $\bE[|X|^p]<\infty$ for some $p\in[1,\infty)$.
\end{enumerate}
Of course, we have $1.\Rightarrow 2.\Rightarrow 3.$ Hence, by Theorem~\ref{thm:existence}, a sufficient condition for the existence of a fixed point of $\cT$ is that every $R_j, j\in[d]$ satisfies one of the latter three properties. Moreover, supposing $\cT$ has a fixed point $\eta$ with $i$-component $\eta_i = \cL(G_i)$, these properties are transferred from $R_j, j\in[d]$ to $G_i$, the subsequent theorem presents a concise summary of these well-known facts:

\begin{theorem}\label{thm:property_transfer}
	Let $i\in[d]$. If for all $j$ with $i\rightarrow j$
	\begin{enumerate}
		\item $|R_{j}|\leq K$ then $|G_i|\leq \frac{1}{1-\gamma}K$,
		\item $\bE[\exp(\beta|R_j|)]<\infty$ then $\bE[\exp(\beta|G_i|)]<\infty$,
		\item $\bE[|R_j|^p]<\infty$ then $\bE[|G_i|^p]<\infty$.
	\end{enumerate}
\end{theorem}

The first property transfer is fundamental in MDP theory, while property transfers of the second type are well-established in perpetuity equation theory (cf. Theorem 2.4.1 in \cite{buraczewski2016stochastic}). The third is common throughout the Distributional RL literature (cf. Remark~\ref{first_remark}). For completeness, a short proof of Theorem~\ref{thm:property_transfer} is presented in Section~\ref{sec:property_transfer}.

\begin{remark}\label{first_remark}
	For $p\geq 1$ let $\sP_p(\bR)\subset \sP(\bR)$ be the subset of $p$-integrable distributions on $\bR$, that is $\sP_p(\bR) = \{\mu\in\sP(\bR)|\int_{\bR}|x|^pd\mu(x)<\infty\}$. Let $\pr_i:\bR^d\rightarrow\bR$ be the $i$-th projection. The $p$-th Wasserstein distance on $\sP_p(\bR)$ is defined as
	$$d_p(\mu_1,\mu_2) = \inf\Big\{\int_{\bR\times\bR}|x-y|^pd\psi(x,y)~\Big|~\psi\in\sP(\bR^2)~\text{with}~\psi\circ\pr^{-1}_i = \mu_i, i=1,2\Big\}^{1/p}.$$
	The space $(\sP_p(\bR),d_p)$ is a complete metric space and so is the product $(\sP_p(\bR)^d, d'_p)$ with $d'_p(\eta,\nu) = \max_{i\in[d]}d_p(\eta_i,\nu_i)$. In \cite{bellemare2017distributional} the following was shown by extending techniques of \cite{rosler1992fixed}: if for all $j\in[d]$ it holds that $\bE[|R_j|^p]<\infty$, that is $\cL(R_j)\in\sP_p(\bR)$, then the restriction $\cT_{|\sP_p(\bR)^d}$ of $\cT$ to $\sP_p(\bR)^d$ maps to $\sP_p(\bR)^d$ and is a $\gamma$-contraction with respect to $d'_p$. Banach's fixed point theorem yields that the operator $\cT_{|\sP_p(\bR)^d}:\sP_p(\bR)^d\rightarrow\sP_p(\bR)^d$ has a unique fixed point $\eta\in\sP_p(\bR)^d$. The latter also follows from our results: if $\bE[|R_j|^p]<\infty$ for all $j$ then $\bE[\log^+|R_j|]<\infty$ for all $j$ and by Theorem~\ref{thm:existence} there exists a unique fixed point $\eta\in\sP(\bR)^d$. Now $\cL(R_j)\in\sP_p(\bR)$ for all $j$ implies $\eta\in\sP_p(\bR)^d\subset\sP(\bR)^d$ by Theorem~\ref{thm:property_transfer}. Note however, that our Theorem \ref{thm:existence} improves upon results of \cite{bellemare2017distributional} based on the $d_p$ metric twofold: Firstly, Theorem \ref{thm:existence} has necessary and sufficient conditions for the existence and uniqueness of fixed points of $\cT$ without moment assumption such as $\bE[|R_j|^p]<\infty$ for some $p\ge 1$. Secondly, the uniqueness of the fixed point is shown in the full space $\sP(\bR)^d$ whereas the earlier results show uniqueness only in the subspace $\sP_p(\bR)^d\subset \sP(\bR)^d$.
\end{remark}

\begin{remark}\label{rem:weak}
    Theorem~\ref{thm:existence} shows that, if it exists, the fixed point $\eta$ of $\cT$ can be obtained as the weak limit of the iterates $\cT^{\circ n}(\nu)$ as $n\to\infty$. This was observed in the distributional RL literature before, see Proposition 4.34 in Chapter~4 of \cite{bdr2022}, and the proof of this part of the theorem shares many similarities with their arguments. 
\end{remark}

Next we focus on asymptotic behavior of tail probabilities $\bP[X>x]$ and $\bP[X<-x]$ as $x\rightarrow\infty$. Each of the three properties stated in Theorem~\ref{thm:property_transfer} above directly yields asymptotic bounds, in the latter two cases by applying Markov's inequality:
\begin{enumerate}
	\item if $|X|\leq K$ then $\bP[X>x] = \bP[X<-x] = 0$ for all $x>K$,
	\item if $\bE[\exp(\beta|X|)]<\infty$ for some $\beta>0$ then
	$$\limsup_{x\rightarrow\infty}\frac{\log \bP[X>x]}{x} \leq -\beta~~\text{and}~~\limsup_{x\rightarrow\infty}\frac{\log\bP[X<-x]}{x} \leq -\beta,$$
	\item  if $\bE[|X|^p]<\infty$ for some $p\in[1,\infty)$ then $\bP[X>x], \bP[X<-x]$ are $\mathrm{O}(x^{-p})$ as $x\to\infty$.
\end{enumerate}
With this at hand, in each of the three cases presented in Theorem~\ref{thm:property_transfer} one can obtain asymptotic bounds for the tail probabilities $\bP[G_i>x]$ and $\bP[G_i<-x]$ as $x\rightarrow\infty$.\\

Besides the three classical properties discussed so far, a further interesting and important property of (the distribution of) a random variable $X$ implying $\bE[\log^+|X|]<\infty$ is that of \emph{regular variation}. Regularly varying distributions represent a broad and adaptable class of heavy-tailed distributions, containing several important distribution families as specific cases, see Example~\ref{ex:regvar}. As mentioned in the introduction, previous applied works have investigated scenarios that involve heavy-tailed reward distributions. Consequently, understanding in-depth how heavy-tailed reward behavior is reflected in the returns is of interest, a question that Theorem~\ref{thm:regulary_varying_tails} answers in the case of regular variation.\\
Regular variation is directly defined in terms of the asymptotic behavior of the tail probabilities $\bP[X>x]$ and $\bP[X<-x]$ as $x\rightarrow\infty$. First, a function $f:(0,\infty)\rightarrow[0,\infty)$ with $f(x)>0$ for all $x>x_0$ is called \emph{regularly varying with index $\beta\in\bR$} if it is measurable and
$$\lim_{x\rightarrow\infty}\frac{f(tx)}{f(x)} = t^{\beta}~~~\text{for every}~t>0.$$
A function $L:(0,\infty)\rightarrow[0,\infty)$ is called \emph{slowly varying} if it is regularly varying with index $\beta=0$, so $L(tx)/L(x)\rightarrow 1$ for every $t>0$. Every regularly function $f$ with index $\beta$ is of the form $f(x) = x^{\beta}L(x)$ for some slowly varying function $L$, see \cite{bct}.\\
Following \cite{buraczewski2016stochastic}, a real-valued random variable $X$ is called \emph{regularly varying with index $\alpha>0$} if
\begin{enumerate}
	\item[a)] the function $x\mapsto \bP[|X|>x]$ is regularly varying with index $-\alpha$ and
	\item[b)] left and right tail probabilities are \emph{asymptotically balanced}: for some $q\in[0,1]$
	$$\lim_{x\rightarrow\infty}\frac{\bP[X>x]}{\bP[|X|>x]} = q~~\text{and}~~\lim_{x\rightarrow\infty}\frac{\bP[X<-x]}{\bP[|X|>x]} = 1 - q.$$
\end{enumerate}
An equivalent definition reads as follows: a real-valued random variable $X$ is regularly varying with index $\alpha>0$ if and only if there exists a slowly varying function $L$ and some $q\in[0,1]$ such that
\begin{equation}\label{eq:reg_vag}
	\lim_{x\rightarrow\infty}\frac{\bP[X>x]}{x^{-\alpha}L(x)} = q~~\text{and}~~\lim_{x\rightarrow\infty}\frac{\bP[X<-x]}{x^{-\alpha}L(x)} = 1 - q.
\end{equation}
In this situation we have $\bP[|X|>x] = \bP[X>x] + \bP[X<-x] \sim x^{-\alpha}L(x)$ as $x\rightarrow\infty$, so $\bP[|X|>x]$ is regularly varying of index $-\alpha$, and the tail probabilities are asymptotically balanced.\\
Similar to the other three properties, regular variation transfers from rewards to returns in the following sense:
\begin{theorem}\label{thm:regulary_varying_tails}
	Let $i\in[d]$, $\alpha>0$ and $L$ a slowly varying function. Suppose for all $j$ with $i\rightarrow j$ there exist $q_j\in[0,1]$ and $c_j\in[0,\infty)$, with $c_j>0$ for at least one $j$, such that
	\begin{equation}\label{eq:condition_regularly_varying}
		\lim\limits_{x\rightarrow\infty}\frac{\bP[R_j>x]}{x^{-\alpha}L(x)} = q_jc_j~~~\text{and}~~~\lim\limits_{x\rightarrow\infty}\frac{\bP[R_j<-x]}{x^{-\alpha}L(x)} = (1-q_j)c_j.
	\end{equation}
	Then with
	$$w_{ij} = \sum_{t=0}^{\infty}(1-\gamma^{\alpha})\gamma^{\alpha t}\bP[I_t=j|I_0=i],$$
	that is $w_{ij} = \bP[I_N=j|I_0=i]$ with $N\sim\geometric(1-\gamma^{\alpha})$ independent of $(I_t)_t$, it holds that
	\begin{equation}\label{eq:reg-vag-formulas}
		\lim_{x\rightarrow\infty}\frac{\bP[G_i>x]}{x^{-\alpha}L(x)} = \frac{\sum_{j=1}^d w_{ij}q_jc_{j}}{1-\gamma^{\alpha}}~~~\text{and}~~~\lim_{x\rightarrow\infty}\frac{\bP[G_i<-x]}{x^{-\alpha}L(x)} = \frac{\sum_{j=1}^d w_{ij}(1-q_j)c_{j}}{1-\gamma^{\alpha}}.
	\end{equation}
	In particular, $G_i$ is regularly varying with index $\alpha$.
\end{theorem}

\begin{remark}
	For $d=1$ Theorem~\ref{thm:regulary_varying_tails} is equivalent to Theorem~2.4.3 (2) in \cite{buraczewski2016stochastic}. For $d=1$ the tail asymptotic of $\bP[|G_1|>x]$ as $x\rightarrow\infty$ could also be directly obtained from \cite{cline1983infinite}, Theorem~2.3.
\end{remark}

\begin{example}\label{ex:regvar}
	A regularly varying random variable $R$ is called Pareto-like if $\bP[R>x]\sim qcx^{-\alpha}$ and ${\bP[R<-x]\sim (1-q)cx^{-\alpha}}$ for some constants $c>0, q\in[0,1], \alpha>0$, that is if in \eqref{eq:reg_vag} one can choose $L$ to be constant. Examples of Pareto-like distributions include Pareto, Cauchy, Burr and $\alpha$-stable distributions, $\alpha<2$, see \cite{ekm}.
	If there exists a non-empty subset $J\subseteq [d]$ and $\alpha>0$ such that for each $j\in J$ $R_j$ is Pareto-like with index $\alpha$ and for each $j\notin J$ $\bP[|R_j|>x] = o(x^{-\alpha})$, then Theorem~\ref{thm:regulary_varying_tails} applies by choosing $L\equiv 1$ and yields that $G_i$ is also Pareto-like.
\end{example}

\begin{remark}[Application in Distributional Dynamic Programming]\label{rem:distributional-policy-evaluation}
	The goal in Distributional Dynamic Programming is to perform distributional policy evaluation on a computer. A major problem to solve is that probability distributions $\eta_s\in\sP(\bR)$ are infinite dimensional objects and hence some form of approximation has to be applied in practice. Several approaches are discussed in Chapter~5 of \cite{bdr2022}. Additional issues arise in cases of Pareto-like tail behaviors with unbounded support as in Example~\ref{ex:regvar}: it is then known that (some of) the true state-value distributions are Pareto-like, which may be an information desirable to include in evaluating the policy. Theorem~\ref{thm:regulary_varying_tails} justifies to model the tails of the state-value distributions a priori as a Pareto-like distribution, the correct choice of asymptotic parameters is given by \eqref{eq:reg-vag-formulas}. We plan to report on such issues in future work.
\end{remark}

\subsection{Connection to multivariate fixed point equations}\label{sec:multivariate-fixed-point-equations}

A prevalent method for examining probabilistic behaviors of systems involves coupling techniques, wherein multiple dependent versions of a system of interest are constructed and their collective behavior is analyzed. In this section, we adapt this approach to our specific context, demonstrating how it allows for the application of results from the well-established theory of multivariate distributional fixed point equations in the study of distributional Bellman equations.\\
A first key insight is that the operator $\cT:\sP(\bR)^d\to\sP(\bR)^d$, defined in \eqref{eq:operator}, depends on the random pairs $(R_1,J_1),\dots,(R_d,J_d)$ only through their marginal laws
\begin{equation}
\big(\cL(R_1,J_1),\dots,\cL(R_d,J_d)\big)\in\sP(\bR\times[d])^d,
\end{equation}
recall the definition of $\cT = (\cT_1,\dots,\cT_d)$ being
\begin{align*}
\cT:\sP(\bR)^d\longrightarrow\sP(\bR)^d,~~~~
\cT\big(\cL(G_1),\dots,\cL(G_d)\big)&= \big(\cL(R_1+\gamma G_{J_1}),\dots,\cL(R_d+\gamma G_{J_d})\big),  
\end{align*}
where in the $i$-th component the random variables $G_1,\dots,G_d,(R_i,J_i)$ are independent. However, the second key insight is that the independence of $G_1,\dots,G_d$ is not necessary for the definition: any random vector $(\tilde G_1,\dots,\tilde G_d)$ having the same marginal laws as $(G_1,\dots,G_d)$ and being independent of $(R_i,J_i)$ satisfies $\cL(R_i+\gamma \tilde G_{J_i}) = \cL(R_i + \gamma G_{J_i})$.\\
By a \emph{coupling} of $(R_1,J_1),\dots,(R_d,J_d)$ we understand any way the random pairs could be defined on a common probability space, thus having a joint distribution
\begin{equation}\label{eq:joint}
    \cL\big((R_1,J_1),\dots,(R_d,J_d)\big)\in\sP((\bR\times[d])^d).
\end{equation}
A coupling \eqref{eq:joint} induces an operator $\cbT$ on $\sP(\bR^d)$ given by
\begin{align*}
\cbT:\sP(\bR^d)\longrightarrow\sP(\bR^d),~~~
\cbT\big(\cL(G_1,\dots,G_d)\big) = \cL\big(R_1+\gamma G_{J_1},\dots,R_d+\gamma G_{J_d}\big),
\end{align*}
where on the right hand side $(G_1,\dots,G_d)$ and $((R_1,J_1),\dots,(R_d,J_d))$ are independent. The two key insights mentioned above yield that $\cbT$ and $\cT$ are related by 
\begin{equation}\label{eq:connection}
	\cbT(\zeta)\circ \pr^{-1}_i = \cT_i(\zeta\circ \pr^{-1}_1,\dots,\zeta\circ \pr^{-1}_d),~~i\in[d], \zeta\in\sP(\bR^d),
\end{equation}
with $\pr_i:\bR^d\to\bR$ being the $i$-th coordinate projection and hence $\zeta\circ\pr_i^{-1}\in\sP(\bR)$ the $i$-th marginal law of $\zeta\in\sP(\bR^d)$. An immediate consequence of \eqref{eq:connection} of interest to us is
\begin{align}\label{eq:fixed}
	\cL(G_1,&\dots,G_d)=\zeta\in\sP(\bR^d)~~\text{is a fixed point of $\cbT$}\notag\\
 &\Longrightarrow~~(\cL(G_1),\dots,\cL(G_d)) = (\zeta\circ\pr_1^{-1},\dots,\zeta\circ\pr_d^{-1})\in\sP(\bR)^d~~\text{is a fixed point of $\cT$}.
\end{align}
Corollary~\ref{cor:multivariate} below implies the following converse of this implication: if $\cT$ has a (unique) fixed point $(\cL(G_1),\dots,\cL(G_d))$ then for \emph{any} coupling~\eqref{eq:joint} the induced operator $\cbT$ has a unique fixed point which, due to \eqref{eq:fixed}, takes the form of a coupling $\cL(G_1,\dots,G_d)$ of $G_1,\dots,G_d$. This justifies the following approach to study (fixed points of) $\cT$: \emph{choose a convenient coupling \eqref{eq:joint} and study its induced operator $\cS$, in particular the marginal laws of a fixed point of $\cbT$}. We use this approach proving Theorem~\ref{thm:regulary_varying_tails}, where it is most convenient to consider $(R_1,J_1),\dots,(R_d,J_d)$ to be independent.\\
One advantage of this coupling approach is that it enables access to numerous results from the field of multivariate fixed point equations, as we explain in the following.\\
Let $\cbT$ be induced by a coupling \eqref{eq:joint}. Define the random $d$-dimensional vector 
$$\bbR = (R_1,\dots,R_d)^T$$
and the random $d\times d$ matrix 
$$\bbJ = (\bbJ_{ij})_{(i,j)\in[d]\times[d]}~~\text{with}~~\bbJ_{ij} = \gamma\cdot 1(J_i=j).$$ 
The coupling is then completely encoded in the joint distribution $\cL(\bbJ,\bbR)\in\sP\big(\bR^{d\times d}\times\bR^d\big)$. Moreover, for any random vector $\bbG = (G_1,\dots,G_d)^T$ it holds 
$$\bbJ\bbG+\bbR = \big(R_1+\gamma G_{J_1},\dots,R_d+\gamma G_{J_d}\big)^T$$
and hence, if $\bbG$ is independent of $(\bbJ,\bbR)$, it holds that 
$$\cbT(\cL(\bbG)) = \cL(\bbJ\bbG+\bbR)$$ 
and $\cL(\bbG)\in\sP(\bR^d)$ is a fixed point of $\cbT$ if and only if 
\begin{equation}\label{eq:multivariate}
    \bbG\ed \bbJ\bbG+\bbR.
\end{equation}
Note that \eqref{eq:multivariate} is a single distributional equation in $\bR^d$, whereas the distributional Bellman equations, see \eqref{eq:fixed_point_system}, are a system of $d$ distributional equations in $\bR$. Multivariate fixed point equations like \eqref{eq:multivariate} have long been studied for general joint laws $\cL(\bbJ,\bbR)\in\sP(\bR^{d\times d}\times\bR^d)$ under various types of assumptions, some of them directly applicable to the situation presented here. A comprehensive overview of the theory of multivariate distributional fixed point equations and its applications, in particular to stochastic difference equations and many important stochastic time-series models, is presented in \cite{buraczewski2016stochastic}.\\
Theorem~\ref{thm:existence} can be used to show the following results in which $(\bbJ^{(t)},\bbR^{(t)})_{t\in\bN_0}$ are iid copies of $(\bbJ,\bbR)$. 

\begin{corollary}\label{cor:multivariate}
	For any operator $\cbT$ induced by a coupling \eqref{eq:joint} the following statements are equivalent 
		\begin{enumerate}
		\item[(i)] $\cbT$ has a (unique) fixed point,
		\item[(ii)] $\cbT^{\circ n}(\zeta)$ converges weakly as $n\rightarrow\infty$ for some (all) $\zeta\in\sP(\bR^d)$,
		\item[(iii)] The infinite series $\bbG = \sum_{t=0}^{\infty}\Big[\prod_{s=0}^{t-1}\bbJ^{(s)}\Big]\bbR^{(t)}$ is almost surely (absolutely) convergent,
		\item[(iv)] $\cT$ has a (unique) fixed point,
		\item[(v)] $\bE[\log^+|R_i|]<\infty$ for all essential $i$.
	\end{enumerate}
	If one and hence all of these hold, then the unique fixed point $\zeta\in\sP(\bR^d)$ of $\cbT$ is given by the law of the infinite vector-valued series $\cL(\bbG)$ in (iii) and the unique fixed point of $\cT$ is the vector of marginal~laws of~$\zeta$.
\end{corollary}

Before we present a short proof we explain how this result fits the known theory of multivariate fixed points equations. Let $|\cdot|$ be the euclidean norm on $\bR^d$. We also write $|\cdot|$ for the induced operator norm on $\bR^{d\times d}$. A result by \cite{brandt1986stochastic} shows that for \emph{any} joint law $\cL(\bbJ,\bbR)\in\sP(\bR^{d\times d}\times\bR^d)$ such that $\bE[\log^+|\bbJ|], \bE[\log^+|\bbR|]<\infty$ and such that (the distribution of) $\bbJ$ has a strictly negative top~Lyapunov exponent, that is 
\begin{equation}\label{eq:Lyapunov}
    \inf_{t\geq 1}\frac{1}{t}\bE\big[\log \big|\prod\nolimits_{s=0}^{t-1} \bbJ^{(s)}\big|\big]~<~0,
\end{equation}
the sum in (iii) converges almost surely and its law is the unique solution to \eqref{eq:multivariate}. We refer to Appendix~E in \cite{buraczewski2016stochastic} for more information about top~Lyapunov exponents. In our case, having $\bbJ_{ij} = \gamma 1(J_i=j)$ and hence $|\prod_{s=0}^{t-1}\bbJ^{(s)}| = \gamma^t$, the top~Lyapunov exponent of $\bbJ$ equals $\log(\gamma) < 0$, so is strictly negative, and $|\bbJ| = \gamma$ implies $\bE[\log^+|\bbJ|]=0<\infty$. The implication (v)$\Rightarrow$(iii) in Corollary~\ref{cor:multivariate} thus says that the condition $\bE[\log^+|\bbR|]<\infty$, which is equivalent to $\bE[\log^+|R_i|]<\infty$ for all $i\in[d]$, can be relaxed to $\bE[\log^+|R_i|]<\infty$ for all essential $i$ to obtain almost sure convergence of the infinite series in (iii) in our case.\\
The implication (ii)$\Rightarrow$(iii) can directly be obtained from Theorem~2.1 in \cite{erhardsson2014conditions} since $|\prod_{s=0}^{t-1}\bbJ^{(s)}| = \gamma^t\rightarrow 0$ (almost surely) in our case. To deduce almost sure convergence from weak convergence in Theorem~\ref{thm:existence} (implication (ii)$\Rightarrow$(iii) there) we present Proposition~\ref{prop:distribution_almost_sure} in Section~\ref{sec:proofs}, the proof of which is based on ideas presented in \cite{kallenberg1997foundations} to prove a version of Kolmogorov's three Series theorem also involving distributional convergence.\\
Besides the top~Lyapunov exponent another important notion in multivariate fixed point equations is that of \emph{irreducibility}: the joint law $\cL(\bbJ,\bbR)\in\sP(\bR^{d\times d}\times\bR^d)$ is called irreducible if the only affine linear subspace $\cH\subseteq\bR^d$ that fulfills $\bP[\bbJ\cH+\bbR\subseteq\cH]=1$ is the complete subspace $\cH=\bR^d$. In case $(\bbJ,\bbR)$ is irreducible, the implication (i)$\Rightarrow$(iii) can be obtained directly from Theorem~2.4 in \cite{bougerol1992strict}. In our situation, irreducibility of $\cL(\bbJ,\bbR)$ is not given in any situation and Corollary~\ref{cor:multivariate} shows that it is also not needed to obtain the implication (i)$\Rightarrow$(iii).

\begin{proof}[Proof of Corollary~\ref{cor:multivariate}]    
	(iv)$\Longleftrightarrow$(v), see Theorem~\ref{thm:existence}.\\
	(iv)$\Longrightarrow$(iii). By Theorem~\ref{thm:existence} the infinite series $\sum_{t=0}^{\infty}\gamma^tR_{I_t,I_{t+1},t}$ is almost surely (absolute) convergent. For each $i\in[d]$ the $i$-th component of the vector-valued infinite series $\sum_{t=0}^{\infty}\Big[\prod_{s=0}^{t-1}\bbJ^{(s)}\Big]\bbR^{(t)}$ has the same law as $\cL\Big(\sum_{t=0}^{\infty}\gamma^tR_{I_t,I_{t+1},t}|I_0=i\Big)$, hence every component converges almost surely (absolutely).\\
	(iii)$\Longrightarrow$(ii). Let $\zeta=\cL(\bbG_0)$ with $\bbG_0$ independent of $(\bbJ^{(t)},\bbR^{(t)})_t$. The $n$-th iteration of $\cbT$ can be represented as
	\begin{equation}
		\cbT^{\circ n}(\zeta) = \cL\Big(\sum_{t=0}^{n-1}\Big[\prod_{s=0}^{t-1}\bbJ^{(s)}\Big]\bbR^{(t)}~~+~~\Big[\prod_{s=0}^{n-1}\bbJ^{(s)}\Big]\bbG_0\Big).
	\end{equation}
	Now $\Big[\prod_{s=0}^{n-1}\bbJ^{(s)}\Big]\bbG_0\rightarrow 0$ almost surely and $\sum_{t=0}^{n-1}\Big[\prod_{s=0}^{t-1}\bbJ^{(s)}\Big]\bbR^{(t)}$ converges almost surely by assumption. So $\cbT^{\circ n}(\zeta)$ converges in distribution, namely to the law of the infinite sum in (iii).\\
	(ii)$\Longrightarrow$(i). Assume $\cbT^{\circ n}(\zeta)\rightarrow \zeta_0$ converges weakly for some $\zeta$. The map $\cbT:\sP(\bR^d)\rightarrow\sP(\bR^d)$ is continuous with respect to the topology of weak convergence, hence $\cbT(\zeta_0) = \lim_n\cbT(\cbT^{\circ n}(\zeta)) = \lim_n\cbT^{\circ(n+1)}(\zeta) = \zeta_0$, so $\zeta_0$ is a fixed point of $\cbT$.\\
	(i)$\Longrightarrow$(iv). See \eqref{eq:connection}.
\end{proof}

\section{Proofs}\label{sec:proofs}

Let $m\in\bN$ and $\gamma\in(0,1)$ be fixed. Consider random variables
\begin{equation}\label{eq:construction-easy}
(K_t)_{t\in\bN_0},~(X_{lt})_{l\in[m]}, t\in\bN_0
\end{equation}
such that $(K_t)_{t}$ is an $[m]=\{1,\dots,m\}$-valued homogeneous Markov chain with an arbitrary starting distribution, $(X_{lt})_{lt}$ is an array of independent real-valued random variables such that for each $l\in[m]$ the variables $(X_{lt})_t$ are iid and $(X_{lt})_{lt}$ and $(K_t)_{t}$ are independent. A state $l$ is called \emph{accessible} if $\bP[K_t=l]>0$ for some $t\in\bN_0$. Note that we allow any starting distribution in our formulation, hence this is not automatically satisfied. A state $l$ is \emph{accessible essential} if and only if $\bP[K_t=l~\text{for infinitely many}~t\in\bN_0]>0$ and this probability equals $1$ if conditioned on the event of at least one visit in $l$. In light of Theorem~\ref{thm:existence}, which we proof in the next subsection, any result shown within the setting \eqref{eq:construction-easy} can be applied to \eqref{eq:construction} by considering
\begin{equation}
[m]\simeq [d]^2, l=(i,j), X_{lt} = R_{ijt}, K_t = (I_t,I_{t+1})
\end{equation}
and in this case working under the condition $\bP[~\cdot~|I_0=i]$ is equivalent to consider the starting distribution $\bP[K_0=(i',j)] = 1(i'=i)p_{ij}$.

\subsection{Existence and Uniqueness: Proof of Theorem~\ref{thm:existence}}\label{sec:existence}

The proof of Theorem~\ref{thm:existence} is based on the following proposition, which shares similarities with Kolmogorov's three-series theorem presented Theorem~5.18 in \cite{kallenberg1997foundations}.
\begin{proposition}\label{prop:distribution_almost_sure}
	For $n\in\bN$ let $S_n = \sum_{t=0}^{n-1} \gamma^t X_{K_t,t}$. Then the following are equivalent
	\begin{itemize}
		\item[(i)] $S_n$ converges in distribution as $n\rightarrow\infty$.
		\item[(ii)] $S_n$ converges almost surely as $n\rightarrow\infty$.
		\item[(iii)] $\bE[\log^+|X_{lt}|]<\infty$ for each accessible essential state $l\in[L]$.
	\end{itemize}
\end{proposition}

\begin{proof}
		Every finite state space Markov chain with probability one finally takes values in one of its irreducible components, hence to show (iii)$\Leftrightarrow$(ii) we can reduce to the case that $(K_t)_t$ is irreducible, in particular every state is accessible essential in this case.\\
		(iii)$\Rightarrow$(ii): We assume $\bE[\log^+|X_{lt}|]<\infty$ for all $l$. Let $0<c<\log(1/\gamma)$ and $\tilde c = \log(1/\gamma)-c>0$. For each $t\in\bN_0$ let
		\begin{equation}\label{eq:rep}
			A_t=\{\gamma^t|X_{K_t,t}|\geq \exp(-ct)\} = \{\log^+|X_{K_t,t}|\geq \tilde c t\}.
		\end{equation}
		We show $\sum_{t=0}^{\infty}\bP[A_t]<\infty$ and apply the Borel--Cantelli lemma. Conditioning on $K_t=l$ and bounding $\bP[K_t=l]\leq 1$ yields
		\begin{align*}
		\sum_{t=0}^{\infty}\bP[A_t] = \bP[A_0] + \sum_{t=1}^{\infty}\bP\Big[\frac{\log^+|X_{K_t,t}|}{\tilde c}\geq t\Big] &\leq \bP[A_0] + \sum_{l=1}^L\sum_{t=1}^{\infty}\bP\Big[\frac{\log^+|X_{l,t}|}{\tilde c}\geq t\Big]\\
		&\leq \bP[A_0] + \sum_{l=1}^L \bE\Big[\frac{\log^+|X_{l,t}|}{\tilde c}\Big] < \infty.
		\end{align*}
		The Borel--Cantelli lemma implies that $A_t^c = \{\gamma^t|X_{K_t,t}|<\exp(-ct)\}$ occurs for all but finitely many $t$ almost surely and hence almost surely $\sum_{t=0}^{\infty}\gamma^t|X_{K_t,t}|<\infty$, since the tail of the infinite sum is a.s. finally dominated by that of the geometric series $\sum_t \exp(-ct)<\infty$. Hence, $S_n$ is absolutely convergent and thus convergent almost surely.\\
		(ii)$\Rightarrow$(iii): We assume there is an accessible essential state $l$ such that $\bE[\log^+|X_{lt}|]=\infty$ and show that $S_n$ diverges almost surely. We closely follow the idea in the proof of \cite{buraczewski2016stochastic}, Theorem 2.1.3 (also \cite{vervaat}, proof of Lemma 1.7) and apply the root test, which states that for any real-valued sequence $(a_t)_{t\in\bN_0}$
		$$\limsup_{t\in\bN_0}|a_t|^{1/t}>1~~\Longrightarrow~~\sum_{t=0}^{\infty} a_t~~\text{diverges},$$
		where 'diverges' means not having a finite limit. Let $\tilde C>1$, $C = \tilde C/\gamma>1$ and
		\begin{align*}
			B_t &= \Big\{\big|\gamma^tX_{K_t,t}\big|^{1/t}>\tilde C\Big\}\\
				&= \big\{|X_{K_t,t}|^{1/t}>C\big \}\\
				&= \{\log^+|X_{K_t,t}|>t\log(C)\}.
		\end{align*}
		If we can show $\bP[\limsup_{t\in\bN_0}B_t] = 1$ the root test applies and yields the almost sure divergence of $S_n$ and hence not (ii). We want to apply the second Borel--Cantelli lemma, but since the events $B_t, t\in\bN_0$ are not independent we have to take a little higher effort. For $l\in[m], j\in\bN$ let $T_{lj}$ be the time of the $j$-th visit of $(K_t)_t$ in $l$ with $T_{lj}=\infty$ in case no such visit occurs. With this it holds that
		\begin{equation}
			\limsup_{t\in\bN_0}B_t = \bigcup_{l=1}^m\limsup_{j\in\bN}\{\log^+|X_{l,T_{lj}}|>T_{lj}\log(C)\}.
		\end{equation}
		For calculating the probability of the right-hand side we can replace $X_{l,T_{lj}}$ with $X_{lj}$ since $(X_{lt})_{lt}$ and $(T_{lj})_{lj}$ are independent and $(X_{lt})_t$ is iid for each $l$. So by defining the event
		$$C_{lj} = \{\log^+|X_{lj}|>T_{lj}\log(C)\}$$
		we have
		$$\bP[\limsup\nolimits_{t\in\bN_0} B_t] = \bP\Big[\bigcup_{l'=1}^L\limsup\nolimits_{j\in\bN} C_{l'j}\Big] \geq \bP[\limsup\nolimits_{j\in\bN} C_{lj}],$$
		where on the right hand side the state $l$ is such that $\bE[\log^+|X_{lt}|]=\infty$ which exists by assumption. We show that the right-hand side is one. Using independence of $(X_{lt})_{t}$ and $(T_{lj})_j$ Fubini's theorem implies
		\begin{equation}\label{eq:cond}
			\bP\big[\limsup\nolimits_{j\in\bN} C_{lj}\big] = \int \bP\Big[\limsup\nolimits_{j\in\bN}\big\{\log^+|X_{lj}|>t_{lj}\log(C)\big\}\Big] d\bP^{(T_{lj})_j}\big((t_{lj})_j\big),
		\end{equation}
        where the integral on the right hand side is with respect to the distribution of $(T_{lj})_j$ under $\bP$. Since $(K_t)_t$ is assumed to be an irreducible Markov chain it holds that $\frac{1}{j}T_{lj}\rightarrow c_l>0$ almost surely as $j\rightarrow\infty$ (where $c_l$ is the inverse of the unique stationary distribution probability for state $l$). Let $(t_{lj})_{lj}$ be a realization of $(T_{lj})_{lj}$ with $t_{lj}\sim j c_l$ as $j\rightarrow\infty$. If we can show that for any such realization it holds that $\sum_{j=1}^{\infty}\bP[\log^+|X_{lj}|>t_{lj}\log(C)] = \infty$ we can apply the second Borel--Cantelli lemma to conclude $\bP\big[\limsup_{j\in\bN}\{\log^+|X_{lj}|>t_{lj}\log(C)\}\big] = 1$. Formula \eqref{eq:cond} and the fact that almost every realization $(t_{lj})_j$ of $(T_{lj})_j$ satisfies $t_{lj}\sim j c_l$ then would yield $\bP[\limsup_{j\in\bN} C_j]=1$ as desired.\\
		Since $t_{lj}\sim jc_l$ for every $\varepsilon>0$ there is some $j_0$ such that $t_{lj}<(1+\varepsilon)j c_l$ for every $j>j_0$ which allows to conclude
		\begin{equation*}
			\sum_{j=1}^{\infty}\bP[\log^+|X_{lj}|>(1+\varepsilon)j c_l\log(C)] = \infty~~\Longrightarrow~~\sum_{j=1}^{\infty}\bP[\log^+|X_{lj}|>t_{lj}\log(C)] = \infty
		\end{equation*}
		With $C' = (1+\varepsilon)c_l\log(C) > 0$ we have
		\begin{equation*}
			\sum_{j=1}^{\infty}\bP[\log^+|X_{lj}|>(1+\varepsilon)j c_l\log(C)] \geq \bE\Big[\Big\lfloor\frac{\log^+|X_{lj}|}{C'}\Big\rfloor\Big] = \infty,
		\end{equation*}
		since by assumption $\bE[\log^+|X_{lj}|]=\infty$.\\
		(ii)$\Rightarrow$(i). Obvious\\
		(i)$\Rightarrow$(iii). We show that if there is some accessible essential state $l$ with $\bE[\log^+|X_{l,t}|]=\infty$ then $S_n$ does not converge in distribution. For that we use two technical Lemmas from \cite{kallenberg1997foundations} which were used to prove a version of Kolmogorov's three series theorem. Let $(X'_{lt})_{lt}$ be an array of independent RVs, independent of $(X_{lt})_{lt}, (K_t)_t$ and distributed as $(X_{lt})_{lt}$. Let $\tilde X_{lt} = X_{lt} - X'_{lt}$. In \cite{kallenberg1997foundations} $\tilde X_{lt}$ is called \emph{symmetrization} of $X_{lt}$. Let $S'_n = \sum_{t=0}^{n-1}\gamma^t X'_{K_t,t}$ and $\tilde S_n = S_n - S_n' = \sum_{t=0}^{n-1}\gamma^t \tilde X_{K_t,t}$. Note that $\tilde S_n$ is not a symmetrization of $S_n$, because $S_n, S_n'$ are not independent as they both depend on $(K_t)_t$. But for any fixed realization $(k_t)_t$ we have $\sum_{t=0}^{n-1}\gamma^t \tilde X_{k_t,t}$ a symmetrization of $\sum_{t=0}^{n-1}\gamma^t X_{k_t,t}$, which is enough for our reasoning.\\
        From $\bE[\log^+|X_{l,t}|]=\infty$ one can conclude that
		$\bE[\log^+|\tilde X_{l,t}|]=\infty$ and hence the already established equivalence (ii)$\Longleftrightarrow$(iii) shows that $\tilde S_n$ does not converge almost surely, in fact we saw that $\tilde S_n$ diverges almost surely. Thus, using independence of $(K_t)_t$ and $(\tilde X_{kt})_{kt}$, Fubini's theorem implies
        $$1 = \bP[\tilde S_n~~\text{diverges as $n\rightarrow\infty$}] = \int \bP\Big[\sum\nolimits_{t=0}^{n-1}\gamma^t \tilde X_{k_t,t}~~\text{diverges as $n\rightarrow\infty$}\Big] d\bP^{(K_t)_t}((k_t)_t),$$
        hence for almost every realization $(k_t)_t$ of $(K_t)_t$ the sequence $\sum_{t=0}^{n-1}\gamma^t \tilde X_{k_t,t}$ diverges almost surely. Theorem~5.17 in \cite{kallenberg1997foundations} about sums of independent symmetric random variables applies and shows that $\big|\sum_{t=0}^{n-1}\gamma^t \tilde X_{k_t,t}\big|\rightarrow\infty$ in probability for almost every $(k_t)_t$. Hence, for almost every realization $(k_t)_t$ and any constant $C>0$ it holds that
		\begin{equation}
			\lim_{n\rightarrow\infty}\bP\Big[\Big|\sum_{t=0}^{n-1}\gamma^t \tilde X_{k_t,t}\Big|>C\Big] = 1.
		\end{equation}
		Since $\sum_{t=0}^{n-1}\gamma^t \tilde X_{k_t,t}$ is a symmetrization of $\sum_{t=0}^{n-1}\gamma^t X_{k_t,t}$, Lemma~5.19 in \cite{kallenberg1997foundations} applies and yields the following bound
		\begin{equation}
			\bP\Big[\Big|\sum_{t=0}^{n-1}\gamma^t \tilde X_{k_t,t}\Big|>C\Big] \leq 2 \bP\Big[\Big|\sum_{t=0}^{n-1}\gamma^t X_{k_t,t}\Big|>C\Big].
		\end{equation}
		Hence, for almost every realization $(k_t)_t$ of $(K_t)_t$ it holds that
		\begin{equation}
			\liminf_{n\rightarrow\infty}\bP\Big[\Big|\sum_{t=0}^{n-1}\gamma^t X_{k_t,t}\Big|>C\Big]~\geq \frac{1}{2}.
		\end{equation}
		Fatou's lemma together with Fubini's theorem yields
		\begin{equation}
			\liminf_{n\rightarrow\infty}\bP\Big[\big|S_n\big|>C\Big]~\geq \frac{1}{2}.
		\end{equation}
		This holds for every $C>0$ and hence the sequence $\cL(S_n), n\in\bN$ is not tight which shows that it does not converge weakly.
\end{proof}

\begin{proof}[Proof of Theorem~\ref{thm:existence}]
	Using the special construction let
	$$S_n = \sum_{t=0}^{n-1}\gamma^t R_{I_t,I_{t+1}, t}.$$
	For any $\nu\in\sP(\bR)^d$ with $i$-th component $\nu_i = \cL(W_i)$ let $W_1,\dots,W_d, (R_{ijt})_{ijt}, (I_t)_t$ be independent. Then the representation
	$$\cT^{\circ n}_i(\nu) = \cL(S_n + \gamma^n W_{I_n}|I_0=i)$$
	holds.
	(i)$\Longleftrightarrow$(iv). Let $\eta\in\sP(\bR)^d$ be a fixed point of $\cT$, hence with $i$-th component $\eta_i=\cL(G_i)$ it holds that
	$$\eta_i = \cT^{\circ n}_i(\eta) = \cL(S_n + \gamma^n G_{I_n}|I_0=i).$$
	Since $\gamma^n G_{I_n}\rightarrow 0$ almost surely this implies that conditioned on $I_0=i$ the sequence $S_n$ converges to $\eta_i$ in distribution. In particular, if a fixed point exists, its $i$-th component has to be the distributional limit of $S_n$ under $I_0=i$ and thus fixed points are unique. If a fixed point exists, then $\cT^{\circ n}(\nu)$ converges to the unique fixed point $\eta$ for any $\nu$, since $\gamma^n W_{I_n}\rightarrow 0$ almost surely holds for any random variables $W_1,\dots,W_d$. Now suppose that $\cT^{\circ n}(\nu)\rightarrow \eta$ holds for some $\nu$. The map $\cT$ is continuous in the product topology on $\sP(\bR)^d$ and hence $\cT(\eta) = \cT(\lim_n\cT^{\circ n}(\nu)) = \lim_n\cT(\cT^{\circ n}(\nu)) = \lim_n\cT^{\circ (n+1)}(\nu) = \eta$. So a fixed point exists which finished the equivalence (i)$\Longleftrightarrow$(iv).\\
	(iv)$\Longrightarrow$(iii). Supposing (iv) yields that conditioned on $I_0=i$ the sequence $S_n$ converges in distribution. Applying Proposition~\ref{prop:distribution_almost_sure} by considering $l=(i,j)$ and $X_{lt} = R_{ijt}$ and $K_t = (I_t,I_{t+1})$ with $\bP[K_0=(i',j)] = 1(i'=i)p_{ij}$ yields the almost sure convergence. Since it holds conditioned for every $i$, it also holds unconditioned, hence (iii).\\
	(iii)$\Longrightarrow$(iv). Almost sure convergence holds also conditioned on $I_0=i$, one can now apply Proposition~\ref{prop:distribution_almost_sure}.\\
	(iii)$\Longleftrightarrow$(ii). Again we apply the equivalence of (ii) and (iii) from Proposition~\ref{prop:distribution_almost_sure}. We only note that in the situation $l=(i,j)$ and $K_t=(I_t,I_{t+1})$ a state $l=(i,j)$ is essential with respect to $(K_t)_t$ iff $i$ is essential with respect to $(I_t)_t$ and $p_{ij}>0$. Also any state $l$ with $p_{ij}>0$ is accessible when starting in $I_0=i$. Hence, one obtains that the infinite sum (iii) in Theorem~\ref{thm:existence} converges almost surely iff $\bE[\log^+|R_{ijt}|]<\infty$ for each essential pair $(i,j)$, that is for each essential $i$ and $j$ with $p_{ij}>0$. Since we defined $\mu_{ij}=\delta_0$ for $p_{ij}=0$ this is equivalent to $\bE[\log^+|R_i|]<\infty$ for every essential $i$.
\end{proof}

\subsection{Property transfer: Proof of Theorem~\ref{thm:property_transfer}}\label{sec:property_transfer}

\begin{proposition}\label{prop:exponential_moments}
	Let $\beta>0$. If $\bE[\exp(\beta |X_{lt}|)]<\infty$ for each accessible state $l$ then the infinite sum $S = \sum_{t=0}^{\infty} \gamma^t X_{K_t,t}$ exists almost surely and it holds that $\bE[\exp(\beta|S|)]<\infty$.
\end{proposition}

\begin{proof}
		The infinite sum exists almost surely due to Proposition~\ref{prop:distribution_almost_sure}, as the existence of exponential moments for each accessible state clearly yields the existence of the logarithmic moment. Using triangle inequality, that $x\mapsto\exp(\beta x)$ is continuous non-decreasing and monotone convergence of expectations yields
		$$\bE[\exp(\beta|S|)] \leq \lim_{n\rightarrow\infty}\bE[\exp(\beta\sum_{t=0}^{n}\gamma^t |X_{K_t,t}|)].$$
		Conditioning on $(K_0,\dots,K_n)$ and using independence of $(K_t)_t$ and $(X_{lt})_{lt}$ yields
		\begin{equation}\label{eq:expmom}
			\bE[\exp(\beta\sum_{t=0}^{n}\gamma^t |X_{K_tt}|)] \leq \bE\Big[\prod_{t=0}^n \bE\big[\exp(\beta|X_{lt}|)^{\gamma^t}\big]_{|l=K_t}\Big].
		\end{equation}
		For each $t\in\bN_0$ we have $\gamma^t\in(0,1)$ and hence $x\mapsto x^{(\gamma^t)}$ is a concave function on $[0,\infty)$. For each accessible state $l$ we have
		$$\bE\big[\exp(\beta|X_{lt}|)^{\gamma^t}\big]\leq \bE\big[\exp(\beta|X_{lt}|)\big]^{\gamma^t}<\infty.$$
		Let $c(\beta) = \max\big\{\bE[\exp(\beta|X_l|)]~|~l\in[m]~\text{accessible}\}$. We have $c(\beta)<\infty$. One can bound the right hand side of \eqref{eq:expmom} by
		$$\prod_{t=0}^n c(\beta)^{\gamma^t} = c(\beta)^{\sum_{t=0}^n\gamma^t}.$$
		Putting all together yields
		$$\bE[\exp(\beta|S|)] \leq \lim_{n\rightarrow\infty}c(\beta)^{\sum_{t=0}^n\gamma^t} = c(\beta)^{1/(1-\gamma)}<\infty.$$
\end{proof}

\begin{proof}[Proof of Theorem~\ref{thm:property_transfer}]
	Consider $G_i = \sum_{t}\gamma^t R_{I_t,I_{t+1},t}$ starting with $I_0=i$. With probability one $(I_t,I_{t+1})$ realizes to a pair $(j,j')$ with $i\rightarrow j$ and $p_{jj'}>0$.\\
	1. We have $\bP[|R_j|\leq K] = \sum_{j'}p_{jj'}\bP[|R_{jj't}|\leq K]$. Hence, if $\bP[|R_{j}|\leq K]=1$ for all $j$ with $i\rightarrow j$ then $\bP[|R_{jj't}|\leq K]=1$ for all $(j,j')$ with $i\rightarrow j$ and $p_{jj'}>0$. This yields $\bP[|R_{I_t,I_{t+1},t}|\leq K|I_0=i] =1$ for every $t$ and hence $|G_i|\leq \sum_{t}\gamma^t |R_{I_t,I_{t+1},t}| \leq K/(1-\gamma)$ almost surely.\\
	2. Due to Theorem~\ref{thm:existence} we can consider the situation of Proposition~\ref{prop:exponential_moments} via $l=(i,j)$, $X_{lt} = R_{ijt}$ and $K_t = (I_t,I_{t+1})$ having starting distribution $\bP[K_0 = (i,j)] = p_{ij}$. In this situation a state $l = (i',j)$ is accessible with respect to $(K_t)_t$ if and only if $i\rightarrow i'$ and $p_{i'j}>0$. So $G_i$ has the same distribution as the almost sure limit $\sum_{t=0}^{\infty}\gamma^t X_{K_t,t}$ and one obtains $\bE[\exp(\beta|G_i|)]<\infty$.\\
	3. We have $\bE[|R_j|^p] = \sum_{j'}p_{jj'}\bE[|R_{jj't}|^p]$ hence $\bE[|R_j|^p]<\infty$ for all $j$ with $i\rightarrow j$ implies $\bE[|R_{jj't}|^p]<\infty$ for all $(j,j')$ with $i\rightarrow j$ and $p_{jj'}>0$. Let $||R_{jj't}||_p = \bE[|R_{jj't}|^p]^{1/p}$ and $K = \max_{(j,j')}||R_{jj't}||_p$ where the maximum runs over all relevant pairs $(j,j')$, hence $K<\infty$. Let the Markov chain $(I_t)_t$ be started in state $I_0=i$. Then $G_{i,T} = \sum_{t=0}^T\gamma^t R_{I_t,I_{t+1},t}\to G_i$ almost surely as $T\to\infty$ and the process $(I_t,I_{t+1})_t$ only visits pairs $(j,j')$ with $||R_{jj't}||_p\leq K<\infty$. Since $||\cdot||_p$ is a norm, for $T'>T$ 
     \begin{align*}
         ||G_{i,T'}-G_{i,T}||_p = \Big\|\sum_{t=T+1}^{T'}\gamma^t R_{I_t,I_{t+1},t}\Big\|_p \leq \sum_{t=T+1}^{T'} ||\gamma^tR_{I_t,I_{t+1},t}||_p \leq K\cdot \sum_{t=T+1}^{T'}\gamma^t\rightarrow 0~~\text{as}~\min\{T,T'\}\to\infty,
     \end{align*} 
     hence $(G_{i,T})_{T\in\bN}$ is a Cauchy sequence with respect to $||\cdot||_p$. The associated $L_p$-space is complete, hence there is a random variable $\tilde G_i$ in $L_p$, that is $||\tilde G_i||_p<\infty$, such that $||G_{i,T}-\tilde G_i||_p\to 0$. This mode of convergence also implies $G_{i,T}\to \tilde G_i$ in probability. The almost sure convergence $G_{i,T}\to G_i$ implies that $G_{i,T}\to G_i$ in probability. Limits with respect to convergence in probability are almost surely unique, hence $\tilde G_i=G_i$ with probability one and $||G_i||_p = ||\tilde G_i||_p <\infty$.
\end{proof}

\subsection{Regularly varying tails: Proof of Theorem~\ref{thm:regulary_varying_tails}}\label{sec:regularly_varying_tails}

Our second main theorem is the result about regular variation of fixed points of $\cT$, see Theorem~\ref{thm:regulary_varying_tails}. As noted before, our proof makes use of the well-developed theory for multivariate fixed point equations 
\begin{equation}\label{eq:multivariate2}
	\bbG\ed \bbJ\bbG + \bbR,~~~\text{$\bbG$ and $(\bbJ,\bbR)$ independent}.
\end{equation}
In particular, a notion of regular variation for \emph{multivariate} distributions, extending the one-dimensional case, has been explored in the context of equations such as \eqref{eq:multivariate2} under various assumptions on the joint law $\cL(\bbJ,\bbR)\in\sP(\bR^{d\times d}\times\bR^d)$. Below, we use the coupling approach explained in~Section~\ref{sec:multivariate-fixed-point-equations} to apply some of the available results in proving Theorem~\ref{thm:regulary_varying_tails}.\\
First we introduce the multivariate notion of regular variation in random vectors and state some facts we use in the proof. We refer the reader to Appendix~C in \cite{buraczewski2016stochastic} for more details on the basic properties of multivariate regular variation. Let $\bar\bR=\bR\cup\{-\infty,\infty\}$ be the extended real numbers and let $\bar\bR^d_{0} = \bar\bR^d\setminus\{0\}$. A $\bR^d$-valued random vector $\bbX$ is called \emph{regularly varying} if there exists a non-null Radon measure $\mu$ on $\bar\bR^d_{0}$, that does not charge infinite points, such that
\begin{equation}\label{eq:multregvag}
\lim_{x\rightarrow\infty}\frac{\bP\big[x^{-1}\bbX \in C\big]}{\bP[|\bbX|>x]} = \mu(C)~~~\text{for every measurable $\mu$-continuity set $C\subseteq\bar\bR^d_0$},
\end{equation}
note that $|\cdot|$ denotes the euclidean norm on $\bR^d$. The measure $\mu$ is called the \emph{limit measure} of regular variation of the random vector $\bbX$. The following are well-known facts about regularly varying random vectors, again we refer to Appendix~C in \cite{buraczewski2016stochastic}
\begin{itemize}
	\item If $\mu$ is the limit measure of regular variation of a random vector $\bbX$, then there is a unique $\alpha>0$, called the \emph{index of regular variation}, such that $\mu(tC) = t^{-\alpha}\mu(C)$ for every $\mu$-continuity set $C$ and $t>0$.
    \item If \eqref{eq:multregvag} is satisfied with index of regular variation $\alpha>0$, then the set $C=\{\bbx\in\bR^d|~|\bbx|>1\}$ is a $\mu$-continuity set and hence, because $\{x^{-1}\bbX\in C\} = \{|\bbX|>x\}$, it holds $\mu(C)=1$ and 
    $$x\mapsto \bP[|\bbX|>x]$$ 
    is a regularly varying function with index $-\alpha$. Further, the set $\{\bbx\in\bR^d|\bbw^T\bbx>1\}$ is a $\mu$-continuity set for each $\bbw\in\bR^d$,
	\item $\bbX$ is regularly varying with index $\alpha>0$ if and only if there exists a distribution $\Xi\in\sP(\bS^{d-1})$ on the unit sphere $\bS^{d-1}=\{\bbx\in\bR^d|~|\bbx|=1\}$ such that for every $t>0$
	\begin{equation}
		\frac{\bP[|\bbX|>tx, \frac{\bbX}{|\bbX|}\in\cdot~]}{\bP[|\bbX|>x]} \overset{w}{\rightarrow} t^{-\alpha}\Xi(\cdot)~~~~\text{as $x\rightarrow\infty$},
	\end{equation}
	where $\overset{w}{\rightarrow}$ is weak convergence of finite measures.
\end{itemize}

The following theorem can be obtained easily as a special case of Theorem~4.4.24 in \cite{buraczewski2016stochastic}, note that the empty product of (random) matrices is the identity matrix.

\begin{theorem}\label{thm:regulary_varying_tails_multivariate_general}
    Let $(R_1,J_1),\dots,(R_d,J_d)$ be a coupling, $\bbR=(R_1,\dots,R_d)^T$ and $\bbJ = (\bbJ_{ij})_{ij}$ with $\bbJ_{ij} = \gamma 1(J_i=j)$, see~Section~\ref{sec:multivariate-fixed-point-equations}. Suppose $\bbR$ is regularly varying with index $\alpha>0$ and limit measure $\mu_{\bbR}$. Then the multivariate distributional fixed point equation \eqref{eq:multivariate2} has a unique solution $\bbG$ which is regularly varying with index $\alpha$. In particular, with $(\bbJ^{(s)})_{s\in\bN_0}$ being iid copies of $\bbJ$, it holds 
	\begin{equation}\label{eq:formula}
    	\lim_{x\rightarrow\infty}\frac{\bP[x^{-1}\bbG\in C]}{\bP[|\bbR|>x]} = \sum_{t=0}^{\infty}\bE\Big[\mu_{\bbR}\Big(\Big\{~\bbx\in\bR^d~\Big|~\Big(\prod\nolimits_{s=0}^{t-1}\bbJ^{(s)}\Big)\bbx\in C~\Big\}\Big)\Big]
	\end{equation}
	for every $\mu_{\bbR}$-continuity set $C$ and 
    \begin{equation}\label{eq:asymptotic}
        \bP[|\bbG|>x] \sim c\cdot \bP[|\bbR|>x]~\text{as}~x\to\infty,
    \end{equation}
    where $c>0$ equals the right hand side of \eqref{eq:formula} evaluated at $C=\{\bbx\in\bR^d~|~|\bbx|>1\}$.
\end{theorem}

\begin{proof}
    Since $C = \{\bbx\in\bR^d~|~|\bbx|>1\}$ is a $\mu$-continuity set that satisfies $\bP[|\bbG|>x] = \bP[x^{-1}|\bbG|\in C]$ the asymptotic expression \eqref{eq:asymptotic} follows from \eqref{eq:formula}. The remaining part of the theorem is a direct consequence of~Theorem~4.4.24 from \cite{buraczewski2016stochastic} by noticing that the additional necessary assumptions stated in that theorem, $\bE[|\bbJ|^{\alpha}]<1$ and $\bE[|\bbJ|^{\alpha+\delta}]<\infty$ for some $\alpha>0$ and $\delta>0$, are satisfied in our situation of interest as $\bbJ_{ij} = \gamma 1(J_i=j)$ and hence $|\bbJ|=\gamma<1$.
\end{proof}

In Theorem~\ref{thm:regulary_varying_tails_multivariate_general} an arbitrary coupling $(R_1,J_1),\dots,(R_d,J_d)$ is considered and the (multivariate) regular variation of $\bbR=(R_1,\dots,R_d)^T$, which is an assumption there, transfers over to the (multivariate) regular variation of $\bbG=(G_1,\dots,G_d)^T$, whose marginal laws are the solutions to the distributional Bellman equations of interest in Theorem~\ref{thm:regulary_varying_tails}. Moreover, multivariate regular variation of $\bbG$ can be used to show univariate regular variation of $G_1,\dots,G_d$ by testing \eqref{eq:multregvag} on sets of the form $C=\{\bbx\in\bR^d~|~\pm \bbe_j^T\bbx>1\}$ with $\bbe_j\in\bR^d$ being unit vectors. Note that Theorem~\ref{thm:regulary_varying_tails} concerns solutions of the distributional Bellman~operator $\cT$. Hence, the result of that theorem only depends on the marginal laws $(\cL(R_1,J_1),\dots,\cL(R_d,J_d))$ and hence Theorem~\ref{thm:regulary_varying_tails_multivariate_general}, which assumes presence of a coupling satisfying certain properties, is \emph{not} directly applicable there. But, as explained in Section~\ref{sec:multivariate-fixed-point-equations}, results about coupled situations can become useful by choosing a convenient coupling $\cL((R_1,J_1),\dots,(R_d,J_d))$ that makes certain arguments work out. In the following proof this will be the case when an independence coupling is considered: independence of $R_1,\dots,R_d$ together with the assumptions of Theorem~\ref{thm:regulary_varying_tails} implies that the random vector $\bbR=(R_1,\dots,R_d)^T$ is regularly varying in the multivariate sense, so Theorem~\ref{thm:regulary_varying_tails_multivariate_general} becomes applicable. We now present technical details:

\begin{proof}[Proof of Theorem~\ref{thm:regulary_varying_tails}]
    Let $(R_1,J_1),\dots,(R_d,J_d)$ be independent. Let $\bbe_j\in\bR^d$ be the $j$-th unit vector and
    $$\bbR = (R_1,\dots,R_d)^T = \sum_{j=1}^d R_j\cdot\bbe_j,$$
    that is we express $\bbR$ as a sum of independent random vectors. In the following, we rely on well-established results from the theory of (multivariate) regular variation which show that regular variation is transferred to sums of \emph{independent} random variables (vectors) each having this property.\\
    First, we determine the asymptotic behavior of $\bP[|\bbR|>x]$ as $x\to\infty$: let $i\in[d]$ be fixed. Only those pairs $(R_j,J_j)$ with $i\rightarrow j$ are relevant for the distribution of $G_i$, hence we assume w.l.o.g. that $i\rightarrow j$ holds for all $j$, that is for every $j\in[d]$ there exists constants $q_j\in[0,1]$ and $c_j\in[0,\infty)$ such that
	\begin{equation}\label{eq:individual}
		\lim\limits_{x\rightarrow\infty}\frac{\bP[R_j>x]}{x^{-\alpha} L(x)} = q_jc_j~~~\text{and}~~~\lim\limits_{x\rightarrow\infty}\frac{\bP[R_j<-x]}{x^{-\alpha} L(x)} = (1-q_j)c_j,
	\end{equation}
	where $\alpha>0$ and $L$ is a slowly varying function. This implies for every $j$
	$$c_j = \lim\limits_{x\rightarrow\infty}\frac{\bP[|R_j|>x]}{x^{-\alpha} L(x)} = \lim\limits_{x\rightarrow\infty}\frac{\bP[R_j^2>x^2]}{x^{-\alpha} L(x)} = \lim\limits_{y\rightarrow\infty}\frac{\bP[R_j^2>y]}{y^{-\alpha/2} \tilde L(y)}~~~\text{with}~~~\tilde L(y) = L(\sqrt{y}).$$
    The function $\tilde L$ is slowly varying which shows that $R_j^2$ is regularly varying with index $\alpha/2$ in case $c_j>0$ and $\bP[R_j^2>x] = o(x^{-\alpha/2})$ as $x\to\infty$ in case $c_j=0$. In such a situation, and since $R_1^2,\dots,R^2_d$ are independent by choice of coupling, standard results from the theory of regular variation apply, see Section~1.3.1 in \cite{mikosch1999regular}, and show
    \begin{align*}
        \bP[R_1^2+\dots+R_d^2>y] &\sim \bP[R_1^2>y]+\dots+\bP[R_d^2>y]~~\text{as}~y\to\infty
    \end{align*}
    and hence, by re-substituting $x^2$ for $y$,  
    \begin{align*}
        \bP[R_1^2+\dots+R_d^2>x^2] &\sim \bP[R_1^2>x^2]+\dots+\bP[R_d^2>x^2] \sim (c_1+\dots+c_d)x^{-\alpha}L(x)~~\text{as}~x\to\infty,
    \end{align*}
    that is 
    \begin{equation}\label{eq:betrag}
	   \lim_{x\to\infty}\frac{\bP[|\bbR|>x]}{x^{-\alpha} L(x)} = c_1+\dots+c_d.
	\end{equation}
    Next, we give an argument to show that $\bbR$ is regularly varying with index $\alpha$ (in the multivariate sense): let $\sgn(R_j)\in\{-1,1\}$ be the sign of $R_j$, so that $|R_j\bbe_j|=|R_j|$ and $\frac{R_j\bbe_j}{|R_j\bbe_j|} = \sgn(R_j)\bbe_j$ in case $R_j\neq 0$. If $c_j>0$ then $R_j\bbe_j$ is regularly varying at index $\alpha>0$ in the multivariate sense, the spectral measure $\Xi_j\in\sP(\bS^{d-1})$ is given by
	$$\Xi_j = q_j\delta_{\bbe_j} + (1-q_j)\delta_{-\bbe_j} = \lim_{x\rightarrow\infty}\bP[\sgn(R_j)\bbe_j\in\cdot~||R_j|>x].$$
	Applying standard arguments, see Lemma~C.3.1 in \cite{buraczewski2016stochastic}, it is easily seen that $\bbR = \sum_{j=1}^dR_j\bbe_j$ (a sum of independent random vectors) is regularly varying at index $\alpha$. Let $\mu_{\bbR}$ be the associated measure. We are now able to apply~Theorem~\ref{thm:regulary_varying_tails_multivariate_general} to the solution $\bbG=(G_1,\dots,G_d)^T$ of the distributional equation \eqref{eq:multivariate2}.\\
    First, we evaluate the limit measure $\mu_{\bbR}$ at some sets of interest: Combining \eqref{eq:individual} and \eqref{eq:betrag}  yields for every $j\in[d]$
	\begin{equation}
		\mu_{\bbR}\Big(\big\{\bbx\in\bR^d~|~\bbe_j^T\bbx>1\big\}\Big) = \lim_{x\rightarrow\infty}\frac{\bP[R_j>x]}{\bP[|\bbR|>x]} = \frac{q_jc_j}{c_1+\dots+c_d}.
	\end{equation}
    As in Theorem~\ref{thm:regulary_varying_tails_multivariate_general}, let $(\bbJ^{(s)})_{s\in\bN_0}$ be a sequence of iid copies of $\bbJ$. For every $t\in\bN_0$ the entry of the random matrix $\prod_{s=0}^{t-1}\bbJ^{(s)}$ at position $(i,j)\in[d]^2$ equals $\gamma^t \cdot 1(I_{it}=j)$, where $(I_{it})_{i\in[d],t\in\bN_0}$ are $[d]$-valued random variables such that for every $i\in[d]$ the sequence $(I_{it})_{t}$ has the same law as $(I_t)_{t}$ under $\bP[~\cdot~|I_0=i]$.\\
	By~Theorem~\ref{thm:regulary_varying_tails_multivariate_general} for every $i\in[d]$ it holds that
	\begin{equation}
		\lim\limits_{x\rightarrow\infty}\frac{\bP[G_i>x]}{\bP[|\bbR|>x]} = \sum_{t=0}^{\infty}\bE\Big[\mu_{\bbR}\Big(\big\{\bbx\in\bR^d~\big|~\bbe_i^T\Big(\prod\nolimits_{s=0}^{t-1}\bbJ^{(s)}\Big)\bbx > 1\big\}\Big)\Big].
	\end{equation}
	Using $\mu_{\bbR}(tC)=t^{-\alpha}\mu_{\bbR}(C)$ for every $t\in\bN_0$ it holds
	\begin{align*}
		\bE\Big[\mu_{\bbR}\Big(\big\{\bbx\in\bR^d~\big|~\bbe_i^T\Big(\prod\nolimits_{s=0}^{t-1}\bbJ^{(s)}\Big)\bbx > 1\big\}\Big)\Big] &= \bE\Big[\mu_{\bbR}\Big(\big\{\bbx\in\bR^d~\big|~\gamma^t \bbe_{I_{it}}^T\bbx > 1\big\}\Big)\Big]\\
		&=\gamma^{\alpha t}\sum_{j\in[d]}\bP[I_t=j|I_0=i] \mu_{\bbR}\Big(\big\{\bbx\in\bR^d~|~\bbe_j^T\bbx>1\big\}\Big)\\
		&=\gamma^{\alpha t}\sum_{j\in[d]}\bP[I_t=j|I_0=i]\frac{q_jc_j}{c_1+\dots+c_d}
	\end{align*}
	and hence
	\begin{align*}
		\lim\limits_{x\rightarrow\infty}\frac{\bP[G_i>x]}{\bP[|\bbR|>x]} &= \sum_{t=0}^{\infty}\gamma^{\alpha t}\sum_{j\in[d]}\bP[I_t=j|I_0=i]\frac{q_jc_j}{c_1+\dots+c_d}\\
		&=\frac{1}{1-\gamma^{\alpha}}\sum_{j\in[d]}\bP[I_N=j|I_0=i]\frac{q_jc_j}{c_1+\dots+c_d},
	\end{align*}
	where $N\sim\geometric(\gamma^{\alpha})$ is independent of $(I_t)_t$. With $w_{ij} = \bP[I_N=j|I_0=i]$ and \eqref{eq:betrag} we have
	\begin{equation}
		\lim\limits_{x\rightarrow\infty}\frac{\bP[G_i>x]}{x^{-\alpha} L(x)} = \frac{\sum_{j\in[d]}w_{ij}q_jc_j}{1-\gamma^{\alpha}}.
	\end{equation}
	The tail asymptotic of $\frac{\bP[G_i<-x]}{x^{-\alpha}L\left(x\right)}$ can be calculated the same way using
	\begin{equation}
		\mu_{\bbR}\Big(\big\{\bbx\in\bR^d~|~\bbe_j^T\bbx < -1\big\}\Big) = \lim_{x\rightarrow\infty}\frac{\bP[R_j<-x]}{\bP[|\bbR|>x]} = \frac{(1-q_j)c_j}{c_1+\dots+c_d}.
	\end{equation}
	To see that $G_i$ is regularly varying with index $\alpha$ note that $\bP[|G_i|>x] = \bP[G_i>x]+\bP[G_i<-x]\sim (\sum_{j=1}^dw_{ij}c_j)x^{-\alpha}L(x)$ is a regularly varying function of index $\alpha$ and the balance equation hold:
	\begin{align*}
		\lim\limits_{x\rightarrow\infty}\frac{\bP[G_i>x]}{\bP[|G_i|>x]} &= \frac{\sum_{j=1}^dq_jw_{ij}c_j}{\sum_{j=1}^dw_{ij}c_j}~~\text{and}\nonumber\\
\lim\limits_{x\rightarrow\infty}\frac{\bP[G_i<-x]}{\bP[|G_i|>x]} &= \frac{\sum_{j=1}^d(1-q_j)w_{ij}c_j}{\sum_{j=1}^dw_{ij}c_j} = 1 - \frac{\sum_{j=1}^dq_jw_{ij}c_j}{\sum_{j=1}^dw_{ij}c_j}.
	\end{align*}
\end{proof}

\section*{Acknowledgments}
We thank three anonymous referees for their careful reading, many constructive remarks and suggestions, as well as for pointing to the related literature \cite{morimura2010nonparametric,chung1987discounted,als_buck}. The referees' comments led to a significant improvement of the manuscript.

The first author was partially supported by the Deutsche Forschungsgemeinschaft (DFG, German Research Foundation) - 502386356.

\bibliographystyle{bibstylefile}
\bibliography{on_solutions_to_distributional_bellmann_equations_arxiv}

\end{document}